\documentclass{article}
\usepackage{ijcai25}

\usepackage{times}
\usepackage{soul}
\usepackage{url}
\usepackage[hidelinks]{hyperref}
\usepackage[utf8]{inputenc}
\usepackage[small]{caption}
\usepackage{graphicx}
\usepackage{amsmath}
\usepackage{amsthm}
\usepackage{booktabs}
\usepackage{algorithm}
\usepackage{algorithmic}
\usepackage{xcolor}
\usepackage{bm}
\usepackage{amssymb}
\usepackage{multirow}
\usepackage{dblfloatfix}
\usepackage{dsfont}
\usepackage{subcaption}
\usepackage{float}
\usepackage{pifont}
\usepackage{threeparttable}
\newcommand{\xmark}{\ding{55}}

\newtheorem{theorem}{Theorem}
\newtheorem{lemma}{Lemma}

\title{ILIF: Temporal Inhibitory Leaky Integrate-and-Fire Neuron for Overactivation in Spiking Neural Networks}

\author{
Kai Sun$^1$
\and
Peibo Duan$^{1*}$\and
Levin Kuhlmann$^{1}$\and
Beilun Wang$^{2}$ \And
Bin Zhang$^{3}$
\affiliations
$^1$Department of Data Science and AI, Monash University, Melbourne, Australia\\
$^2$School of Computer Science and Engineering, Southeast University, China\\
$^3$School of Software, Northeastern University, China\\
\emails
\{kai.sun1, peibo.duan, levin.kuhlmann\}@monash.edu,
beilun@seu.edu.cn, zhangbin@mail.neu.edu.cn
}

\begin{document}

\maketitle
\begin{abstract}
The Spiking Neural Network (SNN) has drawn increasing attention for its energy-efficient, event-driven processing and biological plausibility. To train SNNs via backpropagation, surrogate gradients are used to approximate the non-differentiable spike function, but they only maintain nonzero derivatives within a narrow range of membrane potentials near the firing threshold—referred to as the surrogate gradient support width $\gamma$. We identify a major challenge, termed \textbf{the dilemma of $\gamma$}: a relatively large $\gamma$ leads to overactivation, characterized by excessive neuron firing, which in turn increases energy consumption, whereas a small $\gamma$ causes vanishing gradients and weakens temporal dependencies. To address this, we propose a temporal Inhibitory Leaky Integrate-and-Fire (ILIF) neuron model, inspired by biological inhibitory mechanisms. This model incorporates interconnected inhibitory units for membrane potential and current, effectively mitigating overactivation while preserving gradient propagation. Theoretical analysis demonstrates ILIF’s effectiveness in overcoming the $\gamma$ dilemma, and extensive experiments on multiple datasets show that ILIF improves energy efficiency by reducing firing rates, stabilizes training, and enhances accuracy. The code is available at \url{github.com/kaisun1/ILIF}.


\end{abstract}

\section{Introduction}

The Spiking Neural Network (SNN), recognized as the third generation of neural networks, is distinguished by the simulation of neuronal message passing through spike-based activations \cite{maass1997networks}. Unlike the traditional Artificial Neural Network (ANN), which process information with continuous values, SNN, driven by the dynamic accumulation of membrane potential to trigger discrete activation events, effectively mimics the behavior of biological neurons by propagating discrete spike signals (0 or 1) between neurons \cite{tavanaei2019deep}. This spike-based processing endows SNN with remarkable energy efficiency, especially when deployed on neuromorphic hardware platforms such as Intel's Loihi, IBM's TrueNorth, and Tianjic chips \cite{cai2021neuromorphic}.

A key challenge in gradient-based optimization for training SNNs is its inherent non-differentiability due to the discrete nature of spike transmissions between neurons. To resolve this issue, surrogate gradient (SG) has been introduced to enable backpropagation in SNNs  \cite{neftci2019surrogate}. However, these approximations can cause or worsen two major issues: overactivation and gradient vanishing. Specifically, overactivation occurs when the accumulated membrane potential surpasses twice the threshold, causing excessive spiking that drives up energy consumption and masks essential temporal information. Strategies such as adaptive thresholding, residual membrane potential modulation, and normalization have been proposed to regulate neuronal activity and mitigate overactivation \cite{wei2023temporal,rw:stc-lif,jiangtab}. Gradient vanishing, on the other hand, stems from the mismatch between surrogate gradients and discrete spikes, as well as membrane potential decay, resulting in ineffective backpropagation. Current solutions, such as residual learning and adaptive mechanisms \cite{SEW-ResNet,rw:glif}, enhance gradient propagation and improve training efficiency.

Existing research often treats overactivation and gradient vanishing as separate issues, overlooking the conflicting effects of the SG on both phenomena. As shown in Figure~\ref{fig:subfig1}, given the support width \(\gamma\) in the SG, Figure~\ref{fig:subfig2} demonstrates that a larger \(\gamma\) results in elevated firing rates and reduced accuracy (more details will be presented in Section \ref{subsec: dilemma of gamma}). However, as \cite{huang2024clif} points out, when the threshold remains constant, a smaller \(\gamma\) risks gradient vanishing. In this paper, this contradiction is referred to as \textbf{the dilemma of \(\gamma\)}, which highlights the need to balance the mitigation of overactivation with the preservation of gradient flow. Inspired by the brain’s efficient spike regulation through feedforward and feedback inhibition, which controls excessive activation, our objective is to design a module that mitigates the $\gamma$-induced conflict between overactivation and gradient vanishing.

\begin{figure}[h]
    \centering
    \begin{subfigure}[b]{0.35\textwidth}
        \centering
        \includegraphics[width=\textwidth]{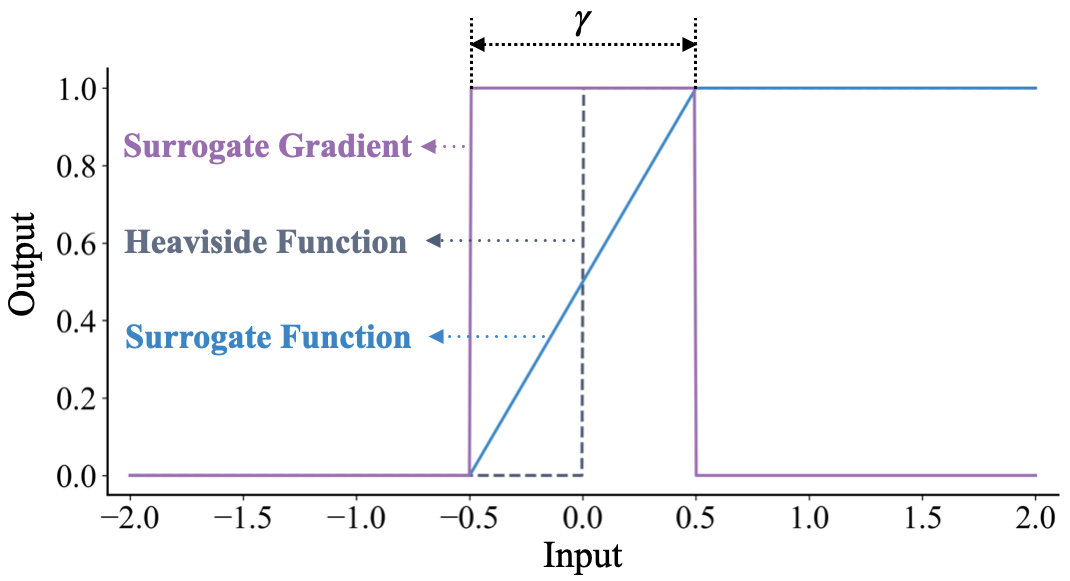}
        \caption{Surrogate gradient}
        \label{fig:subfig1}
    \end{subfigure}
    
    \vspace{0.1cm}
    
    \begin{subfigure}[b]{0.23\textwidth}
        \centering
        \includegraphics[width=\textwidth]{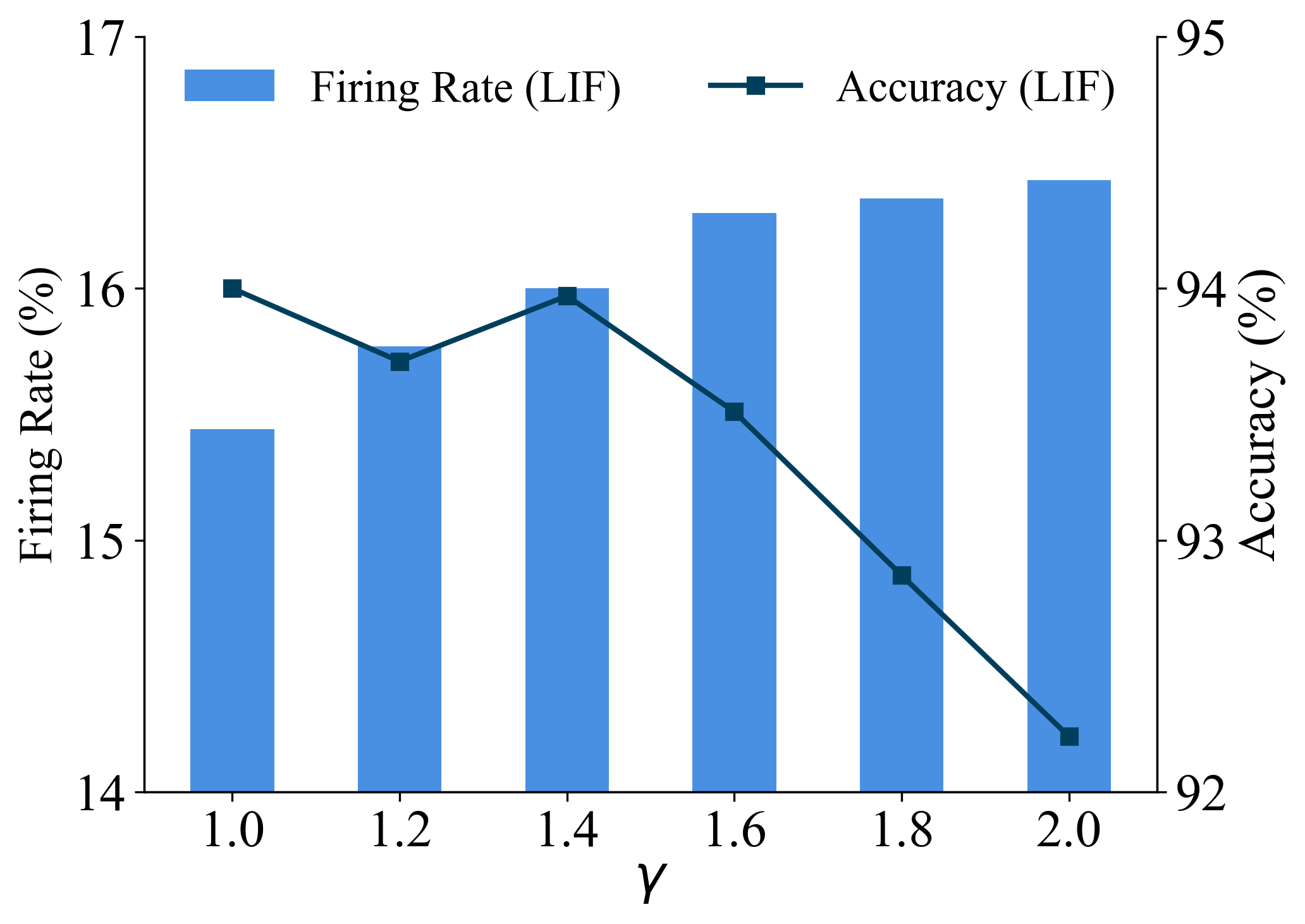}
        \caption{Firing \& accuracy vs. \(\gamma\)}
        \label{fig:subfig2}
    \end{subfigure}
    \hfill
    \begin{subfigure}[b]{0.23\textwidth}
        \centering
        \includegraphics[width=\textwidth]{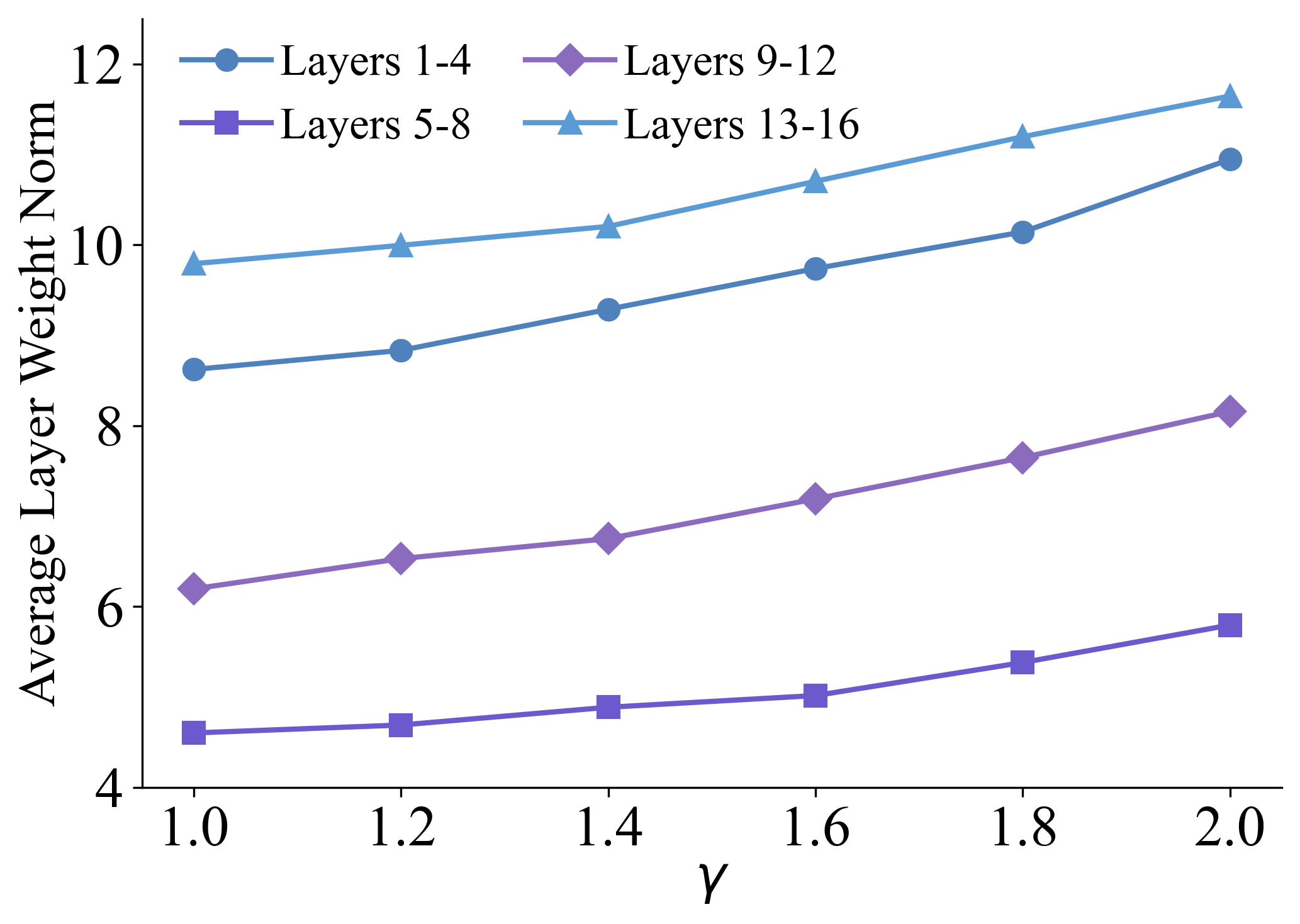}
        \caption{Weight norm vs. \(\gamma\)}
        \label{fig:subfig3}
    \end{subfigure}
    
   \caption{
    Effect of SG support width (\(\gamma\)) on network performance: 
    (a) Surrogate gradient method.
    (b) Changes in firing rate and accuracy with \(\gamma\). 
    (c) Average layer weight norm variation with \(\gamma\).
    }
   
    \label{fig:mainfig}
\end{figure}

This study conducts a theoretical examination of the relationship between $\gamma$ and both excessive activation and the disappearance of gradients. To address the difficulty of balancing overactivation and gradient vanishing solely by adjusting $\gamma$, we propose a temporal Inhibitory Leaky Integrate-and-Fire (ILIF) neuron model with two inhibitory units, namely the membrane potential inhibitory unit (MPIU) and the current inhibitory unit (CIU), which are designed to mimic the inhibitory mechanism of the human brain. The major contributions of this work are as follows:

\begin{itemize}
\item \textbf{Theoretical Analysis:} We conduct an in-depth mathematical investigation to clarify how the configuration of $\gamma$ affects overactivation and its conflict with gradient vanishing, while analyzing how ILIF effectively mitigates both issues.

\item \textbf{SNN Modeling:} 
The proposed MPIU and CIU in the ILIF neuron model function as feedforward and feedback inhibition, temporally interacting to reduce excessive activation and mitigate gradient vanishing.

\item \textbf{Experimental Validation:} Comprehensive experiments demonstrate that the ILIF model’s temporal inhibition mechanism generates fewer spikes while achieving a more stable training process and higher accuracy.
\end{itemize}

\section{Related Work}
\subsection{Overactivation Control}
To mitigate overactivation in SNN, prior works have proposed adaptive thresholds, residual potential modulation, and normalization. Adaptive threshold methods dynamically raise or lower the firing threshold based on recent activity patterns \cite{ding2022biologically,wei2023temporal,fang2020exploiting,zhang2019fast}. Residual modulation methods, such as TC-LIF and CLIF, reduce post-spike membrane potential to suppress reactivation, mimicking the effect of afterhyperpolarization (AHP) \cite{rw:stc-lif,niu2023cirm}. Normalization techniques standardize input distributions across layers and time to stabilize spiking activity \cite{jiangtab,guo2023membrane,duan2022temporal}. However, these methods often rely on handcrafted rules or instantaneous signals, lacking temporal adaptability and biological interpretability. In contrast, our ILIF model integrates inhibitory units that accumulate both short- and long-term activity, and crucially, adjusts inhibition based on the post-spike membrane potential, providing more precise, causal, and biologically aligned suppression than methods like TC-LIF and CLIF, which depend on pre-spike estimates and risk over-inhibition.

\subsection{Improving Gradient Propagation}
Due to the intrinsic properties of spiking neurons—namely, spike-based activation and membrane potential decay—temporal gradients often vanish over time. A key factor contributing to this vanishing is the leakage of membrane potential, which diminishes the impact of earlier inputs as time progresses. To mitigate this, models such as PLIF \cite{rw:plif} adaptively adjust the leakage rate, preserving critical temporal information. Similarly, gating mechanisms employed in models like GLIF \cite{rw:glif}, STC-LIF \cite{rw:stc-lif}, and SpikGRU \cite{dampfhoffer2022investigating} enhance temporal dependencies by regulating information flow within neurons. However, these methods typically incur additional computational overhead and are limited to intra-neuron dynamics. In contrast, our approach introduces biologically inspired temporal connections between inhibitory units, serving as shortcuts that simultaneously facilitate backward gradient propagation and forward inhibitory signaling, thereby enhancing temporal dependencies without increasing model parameters.


\section{Preliminary}
\begin{figure*}[h]
    \centering
    \begin{subfigure}[b]{0.25\textwidth} 
        \centering
        \includegraphics[width=\textwidth]{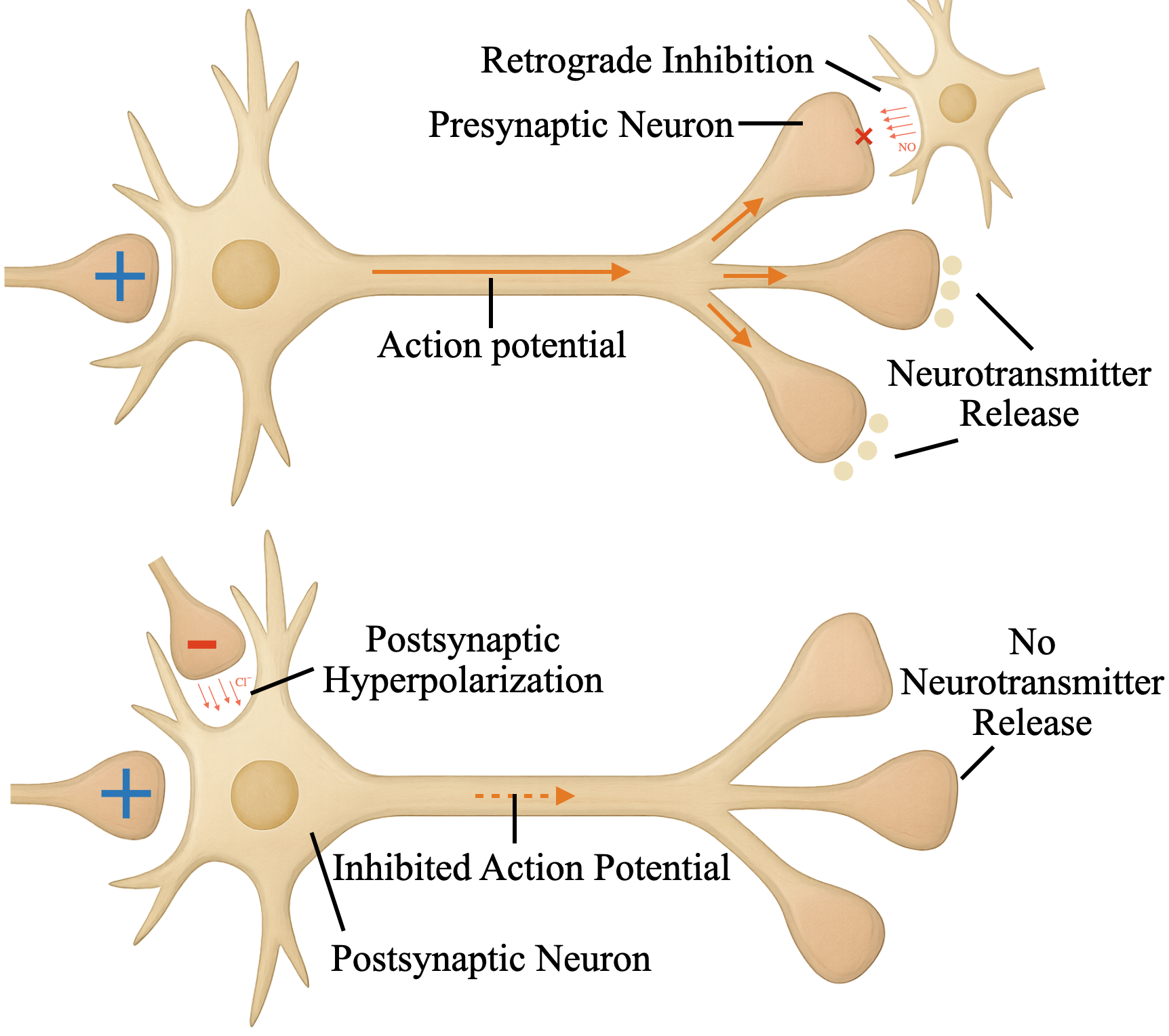}
        \caption{}
        \label{fig:neuron_subfig1}
    \end{subfigure}
    \hfill 
    \begin{subfigure}[b]{0.16\textwidth}
        \centering
        \includegraphics[width=\textwidth]{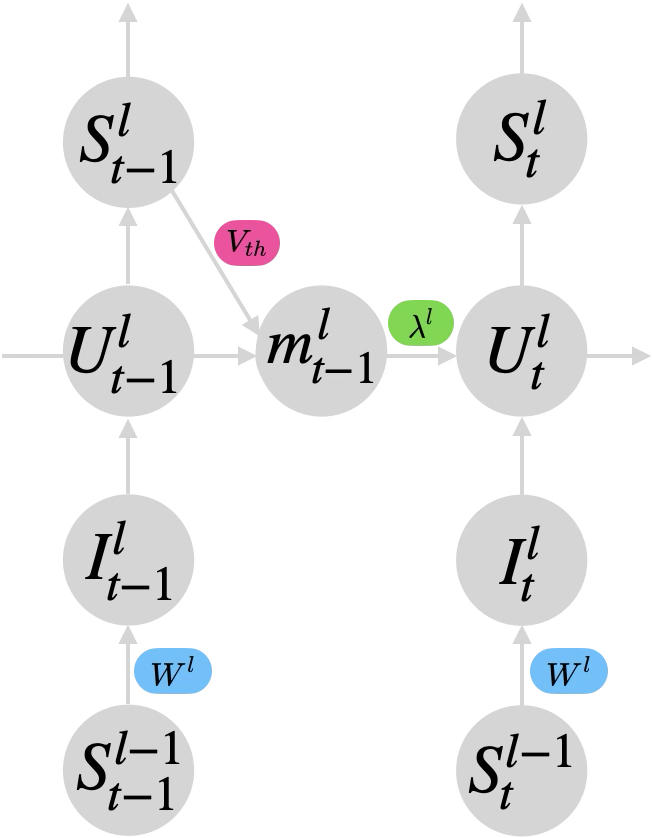}
        \caption{}
        \label{fig:neuron_subfig2}
    \end{subfigure}
    \hfill
    \begin{subfigure}[b]{0.15\textwidth}
        \centering
        \includegraphics[width=\textwidth]{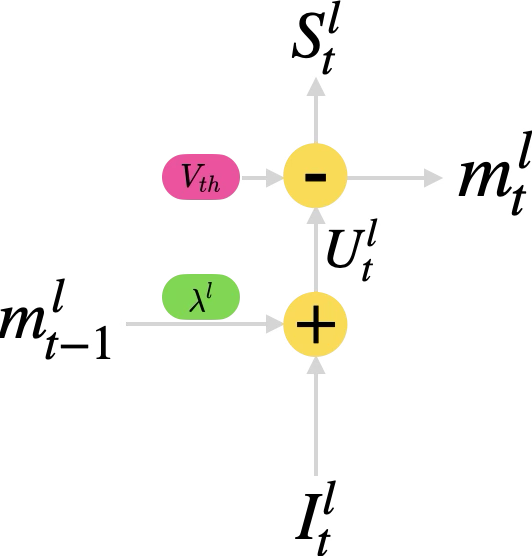}
        \caption{}
        \label{fig:neuron_subfig3}
    \end{subfigure}
    \hfill
    \begin{subfigure}[b]{0.16\textwidth}
        \centering
        \includegraphics[width=\textwidth]{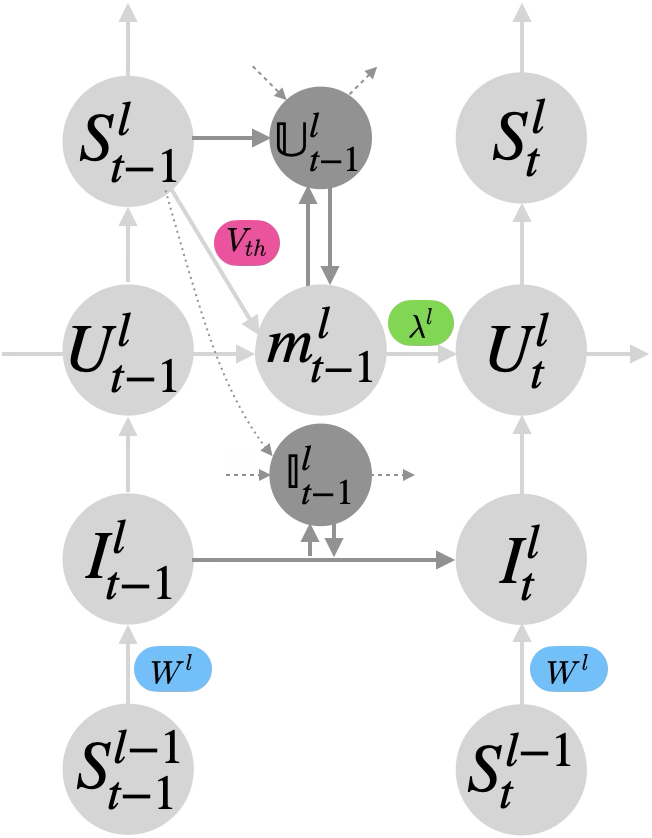}
        \caption{}
        \label{fig:neuron_subfig4}
    \end{subfigure}
    \hfill
    \begin{subfigure}[b]{0.23\textwidth}
        \centering
        \includegraphics[width=\textwidth]{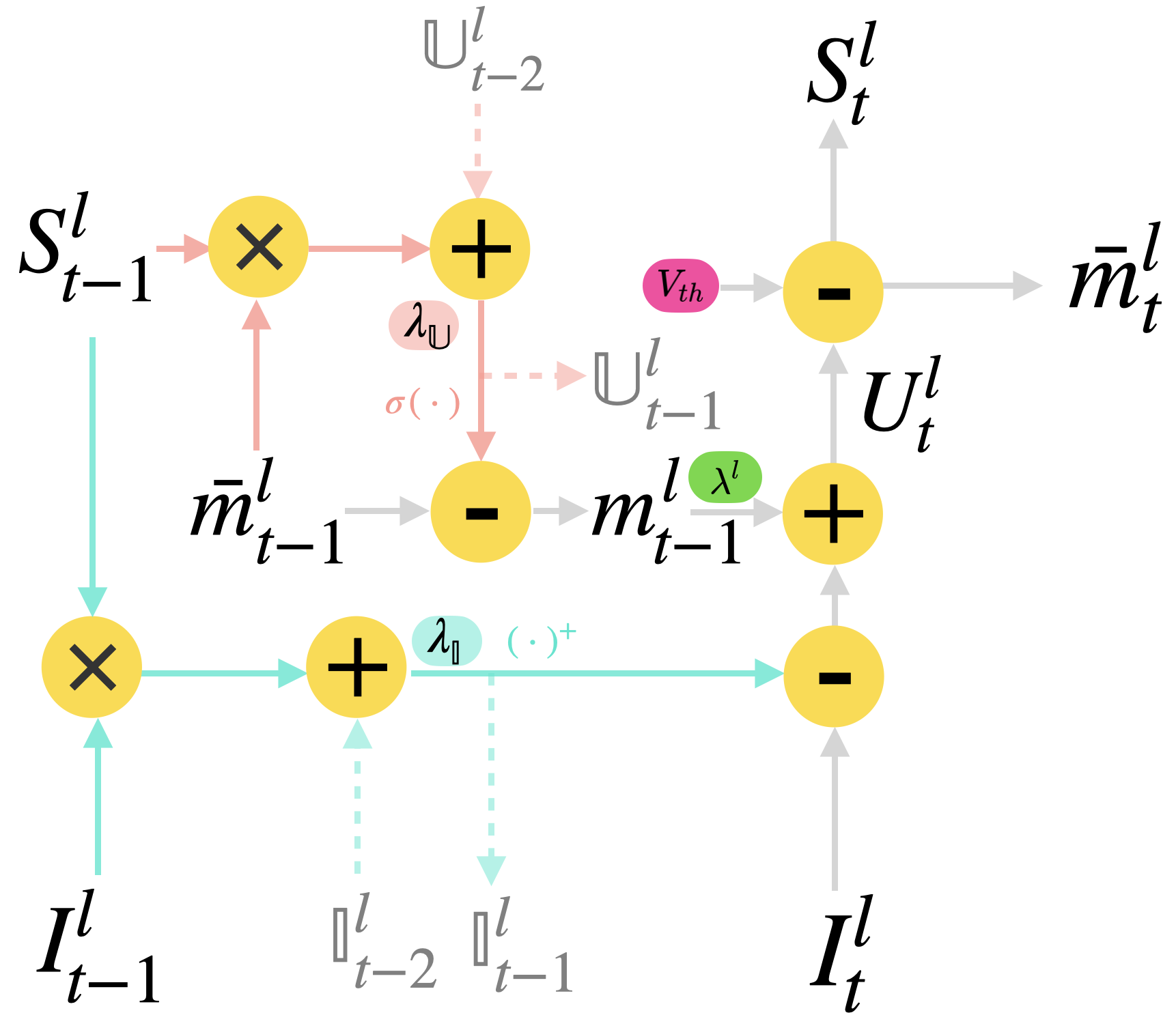}
        \caption{}
        \label{fig:neuron_subfig5}
    \end{subfigure}
    \caption{(a) Diagram of the inhibition mechanism. (b) Structure of the vanilla LIF model. (c) Internal operations of the vanilla LIF model. (d) Structure of the ILIF model. (e) Internal operations of the ILIF model.}
    \label{fig:LIF_ILIF}
\end{figure*}

\subsection{Vanilla LIF Neuron Model}
As shown in Figures~\ref{fig:neuron_subfig2} and \ref{fig:neuron_subfig3}, the LIF serves as a fundamental computational framework, encompassing membrane potential integration, leakage, and spike firing upon exceeding the threshold. The model is described by the following equations:
\begin{align}
    \bm{U}^l[t] &= \lambda \bm{m}^l[t-1] + \bm{I}^l[t] \label{eq:membrane_potential} \\
    \bm{I}^l[t] &= \bm{W}^{l} \bm{S}^{l-1}[t] \label{eq:input_current} \\
    \bm{S}^l[t] &= \mathbb{H}(\bm{U}^l[t] - V_{th}) = 
    \begin{cases} 
    1, & \bm{U}^l[t] \geq V_{th} \\ 
    0, & \text{Otherwise}  
    \end{cases} \label{eq:spike_output} \\
    \bm{m}^l[t] &= \bm{U}^l[t] - \bm{S}^l[t] V_{th} \label{eq:membrane_potential_update}
\end{align}
In the LIF model, $\bm{U}^l[t]$ is the membrane potential in the $l$-th layer at time step $t$, combining the postsynaptic current $\bm{I}^l[t]$ and the residual membrane potential from the previous time step $\bm{m}^l[t-1]$ with a decay factor $\lambda$. The postsynaptic current $\bm{I}^l[t]$ is computed as the product of the synaptic weight $\bm{W}^l$ and the input spike $\bm{S}^{l-1}[t]$. The output spike $\bm{S}^l[t]$ is generated using the Heaviside function $\mathbb{H}(\cdot)$, which outputs a spike ($\bm{S}^l[t] = 1$) when the membrane potential $\bm{U}^l[t]$ exceeds the threshold $V_{\text{th}}$ and no spike ($\bm{S}^l[t] = 0$) otherwise. After firing, the membrane potential $\bm{m}^l[t]$ is softly reset by subtracting $V_{\text{th}}$.

\subsection{Spatio-Temporal Backpropagation}
Training SNN directly involves implementing Backpropagation Through Time (BPTT) \cite{werbos1990backpropagation}  and using SG to handle the non-differentiability of spike signals. The gradient of the loss function $\mathcal{L}$ with respect to the weight $\bm{W}^l$ at layer $l$ is calculated across all time steps $T$. It is given by
\begin{equation}
\frac{\partial \mathcal{L}}{\partial \bm{W}^l} = \sum_{t=1}^T \frac{\partial \mathcal{L}}{\partial \bm{U}^l[t]} \cdot \frac{\partial \bm{U}^l[t]}{\partial \bm{W}^l}, l = L, L - 1, \cdots, 1,
\label{eq:5}
\end{equation}
This gradient is computed by decomposing it into spatial and temporal components:
\begin{equation}
\frac{\partial \mathcal{L}}{\partial \bm{U}^l[t]} = 
\underbrace{\frac{\partial \mathcal{L}}{\partial \bm{S}^l[t]} \cdot \frac{\partial \bm{S}^l[t]}{\partial \bm{U}^l[t]}}_{\text{Spatial Term}} + 
\underbrace{\frac{\partial \mathcal{L}}{\partial \bm{U}^l[t+1]} \cdot \frac{\partial \bm{U}^l[t+1]}{\partial \bm{U}^l[t]}}_{\text{Temporal Term}},
\label{eq:gradient_decompose}
\end{equation}
The \textbf{Spatial Term} represents the gradient contribution from the current time step, while the \textbf{Temporal Term} accounts for the influence of future time steps on the current gradient. It is calculated recursively as follows:

\begin{equation}
\resizebox{1.04\linewidth}{!}{$
\begin{aligned}
    \frac{\partial \mathcal{L}}{\partial \bm{U}^{l}[t]} &= \frac{\partial \mathcal{L}}{\partial \bm{S}^l[t]} \frac{\partial \bm{S}^l[t]}{\partial \bm{U}^{l}[t]} + \sum_{t' = t+1}^{T}  \frac{\partial \mathcal{L}}{\partial \bm{S}^l[t']} \frac{\partial \bm{S}^l[t']}{\partial \bm{U}^{l}[t']} \prod_{t'' = t+1}^{t'} \lambda \epsilon^l[t''-1] \\
    &= \sum_{t' = t}^{T} \frac{\partial \mathcal{L}}{\partial \bm{S}^l[t']} \frac{\partial \bm{S}^l[t']}{\partial \bm{U}^{l}[t']} \prod_{t'' = t+1}^{t'} \lambda \epsilon^l[t''-1]
\end{aligned}
$}
\label{eq:7}
\end{equation}


where the term \(\epsilon^l[t]\) for LIF model is as:
\begin{equation}
\epsilon^l[t] \triangleq \frac{\partial \bm{U}^{l}[t+1]}{\partial \bm{U}^{l}[t]} + \frac{\partial \bm{U}^{l}[t+1]}{\partial \bm{S}^l[t]} \frac{\partial \bm{S}^l[t]}{\partial \bm{U}^{l}[t]}
\label{eq:8}
\end{equation}
Detailed derivations are provided in Appendix~\ref{app:LIF_backprop}. The derivative \(\frac{\partial \bm{S}^l[t]}{\partial \bm{U}^l[t]}\) is approximated with the SG \(H'(\bm{U}^l[t])\), where \(H(\cdot)\) provides a smooth approximation of the Heaviside function. A common choice for the SG is the rectangular function, defined as:
\begin{equation}
    \frac{\partial \bm{S}^l[t]}{\partial \bm{U}^{l}[t]} = H'(\bm{U}^l[t]) =
    \frac1\gamma\mathds{1}\left(\left|\bm{U}^l[t]-V_{\mathrm{th}}\right|<\frac{\gamma}{2}\right),
\label{eq:surrogate}
\end{equation}
where $\mathds{1} (\cdot)$ is the indicator function, and $\gamma$ controls the SG support width. $\gamma$ is typically set to $V_{\text{th}}$ \cite{SLTT}. In this case
\begin{equation}
\epsilon^l[t] \triangleq 1 - V_{\text{th}}H'(\bm{U}^l[t])
\label{eq:11}
\end{equation}

The gradient \(\frac{\partial \mathcal{L}}{\partial \bm{S}^l[t]} \) in Eq.~\eqref{eq:7} varies depending on the layers:
\begin{equation}
    \frac{\partial\mathcal{L}}{\partial\bm{S}^l[t]} =
    \begin{cases} 
        \frac{\partial\mathcal{L}}{\partial\bm{S}^L[t]}, & \text{if } l = \mathcal{L}, \\
        \frac{\partial\mathcal{L}}{\partial\bm{U}^{l+1}[t]} \frac{\partial\bm{U}^{l+1}[t]}{\partial\bm{S}^l[t]}, & \text{if } l = \mathcal{L}-1, \dots, 1.
    \end{cases}
\label{eq:partialL/partialS}
\end{equation}
where \(\frac{\partial\bm{U}^{l+1}[t]}{\partial\bm{S}^l[t]}=\bm{W}^{l+1}\).

\subsection{Inhibitory Mechanism}
In neural systems, various inhibitory mechanisms ensure stable and efficient signal processing. For instance, after a neuron fires an action potential, its membrane potential undergoes afterhyperpolarization. This causes the potential to drop below the resting level, reducing excitability and preventing immediate re-firing. The effect of afterhyperpolarization accumulates over successive spikes, intensifying suppression over time. Additionally, activated downstream neurons send retrograde inhibitory signals to presynaptic neurons by modulating ion channel activity and suppressing neurotransmitter release. This process, mediated by GABAergic interneurons and commonly referred to as retrograde inhibition, suppresses presynaptic currents and limits excessive signal transmission \cite{inhibitory-mechanism-1,inhibitory-mechanism-2}. Inspired by these mechanisms, we enhance the LIF model by incorporating inhibitory processes to reduce overactivation and facilitate more effective gradient flow. The corresponding inhibitory mechanisms are illustrated in Figure~\ref{fig:neuron_subfig1}.

\section{Methodology}
\subsection{The Dilemma of \(\gamma\)}
\label{subsec: dilemma of gamma}

The dilemma of \(\gamma\) refers to the challenge of balancing overactivation and gradient vanishing in neural networks. The correlation between \(\gamma\) and overactivation has been underexplored, motivating our investigation into this correlation.

\begin{lemma}
    \label{lemma: gamma and overactivation}
    The likelihood of experiencing overactivation is positively correlated with \(\gamma\).
\end{lemma}

\begin{proof}
When the neural network converges to an optimum such that no further weight updates occur, we consider the system to have reached an equilibrium. Let \(\bm{W}_1\) be the equilibrium weights when \(\gamma = \gamma_1\) and \(\bm{W}_2\) be the equilibrium weights when \(\gamma = \gamma_2\), with \(\gamma_2 > \gamma_1\). 
Suppose that there are neurons in the firing state (requiring \(\bm{W}^{l} \bm{S}^{l-1}[t] > 0\) to accumulate sufficient membrane potential), whereas others remain inactive (requiring \(\bm{W}^{l} \bm{S}^{l-1}[t] < 0\) to suppress activation).

In the output layer, the gradient terms related to the loss function are:
\(
    \frac{\partial \mathcal{L}}{\partial \bm{S}^L[t]} 
    = \bm{Y}^L[t] - \hat{\bm{Y}}^L[t], 
    \quad 
    \frac{\partial \bm{U}^L[t]}{\partial \bm{W}^L} 
    = \bm{S}^{L-1}[t],
\)
where \(\bm{Y}^L[t]\) represents the average output (within [0,1]), and \(\hat{\bm{Y}}^L[t]\) represents the predicted values encoded in a one-hot format (taking values of either 0 or 1). As derived from  Equations~\eqref{eq:5}, \eqref{eq:7}, \eqref{eq:surrogate} and \eqref{eq:11}, the expected gradient with respect to \(W^L\) is
\begin{equation}
\begin{aligned}
    \frac{\partial \mathcal{L}}{\partial W^L} = 
    & \sum_{t' = t}^{T} 
    \Bigg( (\overbrace{\bm{Y}^L[t']}^{[0,1]} - \hat{\bm{Y}}[t']) 
    \overbrace{H'(U^L[t'])}^{\geq 0} \\ 
     \prod_{t'' = t+1}^{t'} \lambda  &\overbrace{\big( 1 - V_{\text{th}} H'(U^L[t''])}^{\geq 0} \big) \Bigg) 
    \overbrace{\bm{S}^{L-1}[t]}^{\geq 0}
    \begin{cases}
        \leq 0, & \hat{\bm{Y}}[t'] = 1 \\ 
        \geq 0, & \hat{\bm{Y}}[t'] = 0 
    \end{cases}
    \label{eq:derivative_expansion}
\end{aligned}
\end{equation}

Under the SG framework, the gradient for a given sample becomes zero whenever the membrane potential moves outside the range 
$\bigl[V_{\text{th}} - \tfrac{\gamma}{2},\, V_{\text{th}} + \tfrac{\gamma}{2}\bigr]$. Neurons whose membrane potentials exceed or fall below this interval stop receiving weight updates. As $\lambda$ increases from $\lambda_1$ to $\lambda_2$, neurons previously outside the range may re-enter it, allowing further updates during the transition from $\bm{W}_1$ to $\bm{W}_2$.

According to Equation~ \eqref{eq:derivative_expansion}, Neurons expected to fire ($\hat{\bm{Y}}[t'] = 1$) receive negative gradients, which increase positive weights. Conversely, neurons expected to remain inactive ($\hat{\bm{Y}}[t'] = 0$) receive positive gradients, which decrease negative weights. This process contributes to an overall increase in the weight norm $\|\bm{W}\|$. Additionally, based on Equation~\eqref{eq:partialL/partialS} and the chain rule, changes in the final layer’s weights propagate to earlier layers, causing their weights to increase or decrease accordingly.

During activation, the membrane potential increment is \(\bm{W}^l \bm{S}^{l-1}[t]\). A larger \(\|\bm{W}\|\) amplifies this increment, increasing the chance of exceeding twice the threshold and causing overactivation. As shown in Figure~\ref{fig:subfig3}, increasing \(\gamma\) leads to a higher average layer weight norm, which amplifies the membrane potential further, thereby heightening the risk of overactivation.
\end{proof}

\begin{lemma}
    \label{lemma: gradient vanishing and gamma}
    The likelihood of experiencing gradient vanishing is inversely correlated with \(\gamma\).
\end{lemma}
The gradient in Eq.~\eqref{eq:7} includes the product \(\prod \epsilon^l[t]\), where \(\epsilon^l[t] = 1 - V_{\text{th}}/\gamma\) when the membrane potential is near threshold and \(\gamma > V_{\text{th}}\). As \(\gamma\) decreases toward \(V_{\text{th}}\), this factor becomes smaller, reducing the gradient magnitude over time. At \(\gamma = V_{\text{th}}\), \(\epsilon^l[t] = 0\), and gradients are completely blocked. See Appendix~\ref{app:gamma_temporal_impact} for details.

As seen from Lemmas \ref{lemma: gamma and overactivation} and \ref{lemma: gradient vanishing and gamma}, the impact of \(\gamma\) is twofold: increasing \(\gamma\) exacerbates overactivation, while decreasing it increases the risk of gradient vanishing. This trade-off demonstrates the limitation of adjusting \(\gamma\) alone to improve SNN performance. Thus, specialized modules are needed to strike an optimal balance between these competing effects.

\subsection{ILIF Neuron Model}
To address overactivation and gradient vanishing in the vanilla LIF model, we propose the ILIF model with two biologically-inspired inhibitory mechanisms.  The MPIU provides long-term inhibition by mimicking afterhyperpolarization, stabilizing neuronal excitability. The CIU delivers short-term inhibition via retrograde-like feedback, regulating presynaptic currents. The ILIF structure and internal operations are illustrated in Figures~\ref{fig:neuron_subfig4} and \ref{fig:neuron_subfig5}.

\textbf{\textit{MPIU}} Each spike contributes to the integration of the post-spike membrane potential $\bar{\bm{m}}^l[t]$ into $\mathbb{U}^l[t]$, which decays slowly ($\lambda_{\mathbb{U}} \approx 1$) to maintain a long-term memory of spiking activity:
\begin{equation}
\label{eq:MPIU}
\mathbb{U}^l[t] = \lambda_{\mathbb{U}} \left( \mathbb{U}^l[t-1] + \bm{S}^l[t] \cdot \bar{\bm{m}}^l[t] \right)
\end{equation}

This unit accumulates historical firing patterns, retaining the neuron’s firing history and mimicking the afterhyperpolarization phenomenon in biological systems. As firing activity accumulates, the accumulated inhibition \(\mathbb{U}^l[t]\) grows, progressively reducing excitability and preventing excessive firing. This inhibition directly affects the membrane potential before subsequent firings:

\begin{equation}
    \bm{m}^l[t] = \bar{\bm{m}}^l[t] - \bm{S}^l[t] \cdot \sigma\left( \mathbb{U}^l[t] \right)
\end{equation}

\noindent where $\sigma(\cdot)$ is a bounded sigmoid function that allows inhibition to increase with historical firing, in contrast to linear mappings or unbounded nonlinear functions that may lead to instability. This mirrors synaptic plasticity mechanisms that adjust synaptic strengths based on past activity, preventing excessive suppression and maintaining learnability over extended timescales. Moreover, the temporal linkage within MPIU mitigates vanishing gradients by offering a direct pathway for backpropagation, as further analyzed in Section \ref{sec:ILIF analysis}.

\textbf{\textit{CIU}} Complementing the MPIU’s long-term inhibition, the CIU provides rapid, short-term feedback. With a decay coefficient $\lambda_{\mathbb{I}}$ close to 0, primarily inhibiting the incoming current based on the previous time step’s current and the current spikes from the next layer:
\begin{equation}
\label{eq:CIU}
\mathbb{I}^l[t] = \lambda_{\mathbb{I}} \left( \mathbb{I}^l[t-1] + \bm{S}^l[t] \cdot \bm{I}^l[t] \right)
\end{equation}
The incoming current is adjusted as follows:
\begin{equation}
\bm{I}^l[t] = \bm{S}^{l-1}[t] \cdot W^l - \mathbb{I}^l[t-1]^+
\end{equation}
Here, \( \mathbb{I}^l[t-1]^+ \) ensures non-negative inhibition, mimicking the rapid feedback mediated by GABAergic interneurons in biological systems. By delivering rapid inhibition tied to the prior spikes and current, the CIU stabilizes excitability in tandem with MPIU.

Combining these mechanisms, the ILIF model evolves according to the following equations, with the corresponding pseudocode provided in Appendix~\ref{app:pseudocode}:
\begin{equation}
\begin{aligned}
\mathbb{I}^l[t-1] &= \lambda_{\mathbb{I}} \left( \mathbb{I}^l[t-2] + \bm{S}^l[t-1] \cdot \bm{I}^l[t-1] \right) \\
\bm{I}^l[t] &= \bm{S}^{l-1}[t] \cdot W^l - \mathbb{I}^l[t-1]^+ \\
\bm{U}^l[t] &= \lambda^l \bm{m}^l[t-1] + \bm{I}^l[t] \\
\bm{S}^l[t] &= \mathbb{H}\left( \bm{U}^l[t] - V_{\text{th}} \right) = 
\begin{cases}
1, & \bm{U}^l[t] \geq V_{\text{th}} \\
0, & \text{otherwise}
\end{cases} \\
\bar{\bm{m}}^l[t] &= \bm{U}^l[t] - \bm{S}^l[t] V_{\text{th}} \\
\mathbb{U}^l[t] &= \lambda_{\mathbb{U}} \left( \mathbb{U}^l[t-1] + \bm{S}^l[t] \cdot \bar{\bm{m}}^l[t] \right) \\
\bm{m}^l[t] &= \bar{\bm{m}}^l[t] - \bm{S}^l[t] \cdot \sigma\left( \mathbb{U}^l[t] \right)
\end{aligned}
\end{equation}

\begin{table*}[b]
    \centering
    \caption{Comparisons with other SNN neuron models on CIFAR10, CIFAR100, DVSCIFAR10 and DVSGesture}
    \label{table:accuracy}
    \renewcommand{\arraystretch}{1.05} 
    \begin{tabular}{c| l c c c}
        \toprule
        \textbf{Dataset} & \textbf{Method} & \textbf{Network Architecture} & \textbf{Time Step} & \textbf{Accuracy (\%)} \\
        \midrule
        \multirow{11}{*}{\vspace{5em} CIFAR10} 
        & STBP-tdBN \cite{STBP-tdBN} & ResNet-19 & 4 / 6 & 92.92 / 93.16 \\
        & Dspike \cite{Dspike} & Modified ResNet-18 & 4 / 6 & 93.66 / 94.05 \\
        & GLIF \cite{rw:glif} & ResNet-18 & 4 & 94.67 \\
        & SML \cite{SML} & ResNet-18 & 6 & 95.12 \\
        & CLIF \cite{huang2024clif} & ResNet-18 & 4 / 6 & 94.89 / 95.41 \\
        \noalign{\global\arrayrulewidth=0.2pt}\cline{2-5}\noalign{\global\arrayrulewidth=0.5pt} 
        & \textbf{Ours} & ResNet-18 & 4 / 6 & \textbf{95.24 / 95.49} \\
        \midrule
        \multirow{6}{*}{CIFAR100} 
        & STBP-tdBN \cite{STBP-tdBN} & ResNet-19 & 4 / 6 & 70.86 / 71.12 \\
        & Dspike \cite{Dspike} & Modified ResNet-18 & 4 / 6 & 73.35 / 74.24 \\
        & GLIF \cite{rw:glif} & ResNet-18 & 4 / 6 & 76.42 / 77.28 \\
        & SML \cite{SML} & ResNet-18 & 6 & 78.00 \\
        & CLIF \cite{huang2024clif} & ResNet-18 & 4 / 6 & 77.00 / 78.36 \\
        \noalign{\global\arrayrulewidth=0.2pt}\cline{2-5}\noalign{\global\arrayrulewidth=0.5pt} 
        & \textbf{Ours} & ResNet-18 & 4 / 6 & \textbf{77.43 / 78.51} \\
        \midrule
        \multirow{5}{*}{\vspace{0em} DVSCIFAR10}
        & STBP-tdBN \cite{STBP-tdBN} & ResNet-19 & 10 & 67.8 \\
        & Dspike \cite{Dspike} & ResNet-18 & 10 & 75.40 \\
        & OTTT \cite{OTTT} & VGG-11 & 10 & 76.27 \\
        & SLTT \cite{SLTT} & VGG-11 & 10 & 77.17 \\
        \noalign{\global\arrayrulewidth=0.2pt}\cline{2-5}\noalign{\global\arrayrulewidth=0.5pt} 
        & \textbf{Ours} & VGG-11 & 10 & \textbf{78.60} \\
        \midrule
        \multirow{5}{*}{\vspace{0em} DVSGesture}
        & STBP-tdBN \cite{STBP-tdBN} & ResNet-17 & 40 & 96.87 \\
        & OTTT \cite{OTTT} & VGG-11 & 20 & 96.88 \\
        & SLTT \cite{SLTT} & VGG-11 & 20 & 97.92 \\
        & CLIF \cite{huang2024clif} & VGG-11 & 20 & 97.92 \\
        \noalign{\global\arrayrulewidth=0.2pt}\cline{2-5}\noalign{\global\arrayrulewidth=0.5pt} 
        & \textbf{Ours} & VGG-11 & 20 & \textbf{97.92} \\
        \bottomrule
    \end{tabular}
\end{table*}
\subsection{Analysis of the ILIF Model}
\label{sec:ILIF analysis}
The ILIF model builds directly upon the vanilla LIF model by introducing inhibitory decay factors $\lambda_{\mathbb{U}}$ and $\lambda_{\mathbb{I}}$. Setting these factors to zero removes all inhibitory effects, reducing the ILIF model to the standard LIF model. This demonstrates that the ILIF model’s inhibitory mechanisms are natural extensions that refine the neuron's behavior without altering its fundamental structure. In this analysis, we explore the relationship between the ILIF and LIF models and the underlying principles that enable the ILIF model to overcome the limitations of the LIF model.

\begin{theorem}
\label{theorem:inhibition-->large weight}
     \(\bm{W}'\) is the equilibrium weight in the vanilla LIF, and \(\bm{W}''\) is the new equilibrium weight in the ILIF, \(\| \bm{W}'' \| > \| \bm{W}' \|\).
\end{theorem}
\begin{proof}
In the vanilla LIF model, the gradient of the loss \(\mathcal{L}\) with respect to \(\bm{U}^{l}[t]\) consists of both a spatial and a temporal term (denoted as \(a(\bm{W})\), highlighted in \textcolor{cyan!100}{blue}). In contrast, the ILIF model introduces additional gradient terms (denoted as \(b(\bm{W})\), highlighted in \textcolor{red!100}{red}), as shown in Equation~(\ref{eq:ILIF gradient}). A detailed derivation is provided in Appendix~\ref{app:ILIF_derivation}.

\begin{equation}
\begin{aligned}
    \frac{\partial \mathcal{L}}{\partial \bm{U}^{l}[t]} &= 
    \textcolor{cyan!100}{
        \frac{\partial \mathcal{L}}{\partial \bm{S}^l[t]} \frac{\partial \bm{S}^l[t]}{\partial \bm{U}^{l}[t]}} \\
    &\textcolor{cyan!100}{+ \sum_{t' = t+1}^{T} \left( 
        \frac{\partial \mathcal{L}}{\partial \bm{S}^l[t']} \frac{\partial \bm{S}^l[t']}{\partial \bm{U}^{l}[t']} 
        \prod_{t'' = t+1}^{t'} \lambda \epsilon^l[t''-1] 
    \right)}
     \\ 
    & \textcolor{red!100}{ 
        + \sum_{t' = t+1}^{T} \left( 
            \phi^l[t'] \prod_{t'' = t+1}^{t'} \lambda \epsilon^l[t''-1] 
        \right) 
        + \phi^l[t]
    }
\label{eq:ILIF gradient}
\end{aligned}
\end{equation}

where \(\phi^l[t]\) encapsulates the derivatives associated with MPIU and CIU:
\begin{equation}
\begin{aligned}
    \phi^l[t] &= \frac{\partial \mathcal{L}}{\partial \mathbb{U}^{l}[t]}
    \left(
    \frac{\partial \mathbb{U}^l[t]}{\partial \bar{\bm{m}}^{l}[t]}
    \epsilon^l[t] +
    \frac{\partial \mathbb{U}^l[t]}{\partial \bm{S}^l[t]}
    \frac{\partial \bm{S}^l[t]}{\partial \bm{U}^l[t]}
    \right) \\
    &+
    \frac{\partial \mathcal{L}}{\partial \mathbb{I}^{l}[t]}
    (\frac{\partial \mathbb{I}^l[t]}{\partial \bm{S}^l[t]}
    \frac{\partial \bm{S}^l[t]}{\partial \bm{U}^l[t]})
\end{aligned}
\end{equation}


Consider the vanilla LIF model without inhibition, where the equilibrium weight $\bm{W}'$ satisfies $a(\bm{W}') = 0$. The function $a(\bm{W})$ includes terms that are products of $\bm{W}$ and typically oppose the sign of $\bm{W}$. Introducing an additional gradient $b(\bm{W})$, which shares the same sign as $\bm{W}$ (see Appendix~\ref{app:ILIF_gradient_analysis}), modifies the total gradient to $a(\bm{W}) + b(\bm{W})$, and the new equilibrium $\bm{W}''$ must satisfy $a(\bm{W}'') + b(\bm{W}'') = 0$.

If $\bm{W}'' > 0$, then $b(\bm{W}'') > 0$ and thus $a(\bm{W}'') < 0$ to maintain equilibrium. Since $a(\bm{W})$ becomes more negative as $\bm{W}$ increases, this implies $\bm{W}'' > \bm{W}'$. Similarly, if $\bm{W}'' < 0$, then $b(\bm{W}'') < 0$ and $a(\bm{W}'') > 0$, and because $a(\bm{W})$ becomes more positive as $\bm{W}$ decreases, it follows that $\bm{W}'' < \bm{W}'$.

Therefore, in both cases, the magnitude of the new equilibrium increases: $\|\bm{W}''\| > \|\bm{W}'\|$. Empirical results supporting this conclusion are provided in Section~\ref{sec:Effectiveness of Overactivation Inhibition}.
\end{proof}
\begin{theorem}
\label{theorem:large weight-->strong inhibitory effect}
    \(r'\) is the equilibrium firing rate in the vanilla LIF, and \(r''\) is the new equilibrium firing rate in the ILIF. It satisfies that \(r'' < r'\) even \(\|\bm{W}''\| > \|\bm{W}'\|\).
\end{theorem}
\begin{proof}
According to Lemma~\ref{lemma: gamma and overactivation}, we define the excitatory input function \( F(\|\mathbf{W}\|) \), which increases with \( \|\mathbf{W}\| \), indicating that larger weight magnitudes result in higher excitatory input necessary for neuron firing, and the response function \( R(x) \), which maps input \( x \) to firing rate \( r \), reflecting that higher inputs lead to increased activation. Additionally, based on Equations~\eqref{eq:MPIU} and~\eqref{eq:CIU}, we define the inhibition function \( J(r) \), which increases with the firing rate \( r \) and satisfies \( J(r) \geq 0 \) for all \( r \geq 0 \), ensuring that higher firing rates induce stronger inhibitory effects.

In the vanilla LIF without inhibition, the equilibrium firing rate is \( r' = R(F(\|\bm{W}'\|)) \). When inhibition is introduced, the new equilibrium rate satisfies \( r'' = R(F(\|\bm{W}''\|) - J(r'')) \). By Theorem~\ref{theorem:inhibition-->large weight}, we know that \( \|\bm{W}''\| > \|\bm{W}'\| \).

To prove \( r'' < r' \), suppose instead that \( r'' \ge r' \). Since \( R \) is increasing, we have
\begin{equation}
    F(\|\bm{W}''\|) - J(r'') \ge F(\|\bm{W}'\|),
\end{equation}
which implies \( F(\|\bm{W}''\|) - F(\|\bm{W}'\|) \ge J(r'') > 0 \).
Define the function \( H(r) = R(F(\|\bm{W}''\|) - J(r)) \). Since both \( R \) and \( J \) are strictly increasing, we have
\begin{equation}
    \frac{dH}{dr} = R'\bigl(F(\|\bm{W}''\|) - J(r)\bigr)\,(-J'(r)) < 0,
\end{equation}
so \( H(r) \) is strictly decreasing. Note that \( J(r') = 0 \), so
\begin{equation}
H(r') = R(F(\|\bm{W}''\|)) > R(F(\|\bm{W}'\|)) = r'.
\end{equation}
By the fixed-point theorem\footnote{Let \( f(x) \) be strictly decreasing. If \( f(x_1) > x_1 \), then any fixed point \( x^* \) such that \( f(x^*) = x^* \) must satisfy \( x^* < x_1 \).}, since \( H(r') > r' \) and \( H \) is decreasing, the fixed point \( r'' = H(r'') \) must satisfy \( r'' < r' \), contradicting our assumption.

Therefore, \( r'' < r' \). This shows that even though \( \|\bm{W}''\| > \|\bm{W}'\| \), the resulting firing rate is strictly lower, meaning that inhibition effectively balances the excitatory input and prevents overactivation.
\end{proof}

\begin{theorem}
    ILIF introduces additional gradient propagation pathways, mitigating the vanishing gradient problem.
\end{theorem}
\begin{proof}
According to Lemma \ref{lemma: gradient vanishing and gamma}, the vanishing gradient problem in LIF is caused by the term \( \epsilon^l[t] \), which gradually accumulates to zero with lower \(\gamma\). In ILIF, however, membrane potential inhibitory units introduce a direct shortcut that connects all time steps. Because each inhibitory unit is interconnected and the decay coefficient is set to 1 in our experiments, there exists a shortcut in the backward pass given by \(\frac{\partial \mathbb{U}^{l}[t]}{\partial \bm{S}^{l}[t]}
\frac{\partial \bm{S}^{l}[t]}{\partial \bm{U}^{l}[t]}
\) enabling each time step’s inhibitory unit to transmit the gradient directly to the preceding time step. Consequently, as inhibition stops after \(T\) with the output produced, the complete chain of partial derivatives from time step \(T-1\) back to \(t\) can be written as:
\begin{equation}
\begin{aligned}
&\frac{\partial \mathcal{L}}{\partial \mathbb{U}^{l}[T-1]}
\frac{\partial \mathbb{U}^{l}[T-1]}{\partial \mathbb{U}^{l}[T-2]}
...
\frac{\partial \mathbb{U}^{l}[t+1]}{\partial \mathbb{U}^{l}[t]}
\frac{\partial \mathbb{U}^{l}[t]}{\partial \bm{S}^{l}[t]}
\frac{\partial \bm{S}^{l}[t]}{\partial \bm{U}^{l}[t]} \\
&= \frac{\partial \mathcal{L}}{\partial \mathbb{U}^{l}[T-1]}
\frac{\partial \mathbb{U}^{l}[T-1]}{\partial \bm{S}^{l}[t]}
\frac{\partial \bm{S}^{l}[t]}{\partial \bm{U}^{l}[t]}
\end{aligned}
\end{equation}

Since there is no decay in these shortcuts, the gradient can flow from the time step \(T-1\) back to any earlier step without attenuation. This mechanism mitigates the vanishing gradient problem and facilitates effective gradient propagation.
\end{proof}

\setcounter{figure}{3}
\begin{figure*}[t]
    \centering
    \includegraphics[width=1\textwidth]{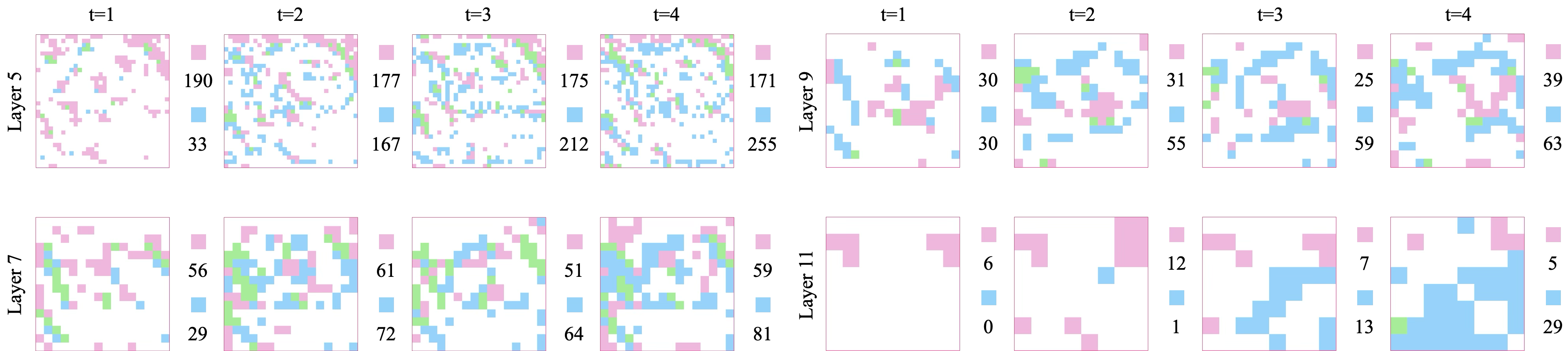} 
    \caption{Comparison of firing rates between LIF and ILIF in Layers 5, 7, 9, and 11 on CIFAR10. Blue pixels indicate LIF spikes, red pixels indicate ILIF spikes, and green pixels indicate simultaneous spikes. Spike counts are displayed beside each time step.}
    \label{fig:Firing_comparison}
\end{figure*}

\section{Experiments}
\subsection{Experimental Settings}
In this research, we conduct experiments on two types of datasets without data augmentation: standard image classification datasets (CIFAR10 and CIFAR100) and neuromorphic datasets (DVSCIFAR10 and DVSGesture). For a comprehensive and unbiased evaluation, the proposed ILIF model is implemented under various settings, including different network architectures, neuron models, and training time steps. Further details are provided in Appendix~\ref{app:experiment_setting}.

\subsection{Evaluation of Accuracy}
As shown in Table~\ref{table:accuracy}, our ILIF model consistently achieves superior top-1 accuracy compared to existing methods. Specifically, our method achieves 95.49\% on CIFAR10 and 78.51\% on CIFAR100 using the ResNet-18 architecture with 6 time-steps. For the neuromorphic dataset, ILIF achieves an accuracy of 78.60\% on DVSCIFAR10 and matches the top-1 accuracy of 97.92\% achieved by other models on DVSGesture, further demonstrating its effectiveness. More experiments are provided in Appendix~\ref{app:additional_experiment}.

\subsection{Effectiveness of Overactivation Inhibition}
\label{sec:Effectiveness of Overactivation Inhibition}
We quantified the proportion of neurons activated for more than half the time. As shown in Figure~\ref{fig:continous_overhalf}, the ILIF model consistently exhibits lower firing rates and fewer instances of continuous activation across all layers compared to LIF model on DVSCIFAR10 and DVSGesture. Furthermore, the reduction in continuous activation rates leads to a lower overall average firing rate, especially with the activation rate by over 30\% on DVSGesture. Additional figures for all datasets are provided in the Appendix~\ref{app:continous_firing_rate}.

To better understand the firing patterns across time steps and layers, we compared the spike counts of the LIF and ILIF models. According to Theorems~\ref{theorem:inhibition-->large weight} and \ref{theorem:large weight-->strong inhibitory effect}, inhibitory units lead to larger synaptic weights but also induce strong inhibition. As shown in Figure~\ref{fig:Firing_comparison}, the ILIF model exhibits a higher firing rate in the first time step due to the absence of immediate inhibition and larger synaptic weights. This efficiently captures critical information early, avoiding prolonged accumulation and delayed activation as observed in the LIF model, where the number of activations rises rapidly across time steps. In later time steps, firing rates decrease as inhibition suppresses redundant activations, ensuring efficient information transfer.

\setcounter{figure}{2}
\begin{figure}[h]
    \centering
    \begin{subfigure}[b]{0.24\textwidth}
        \centering
        \includegraphics[width=\textwidth]{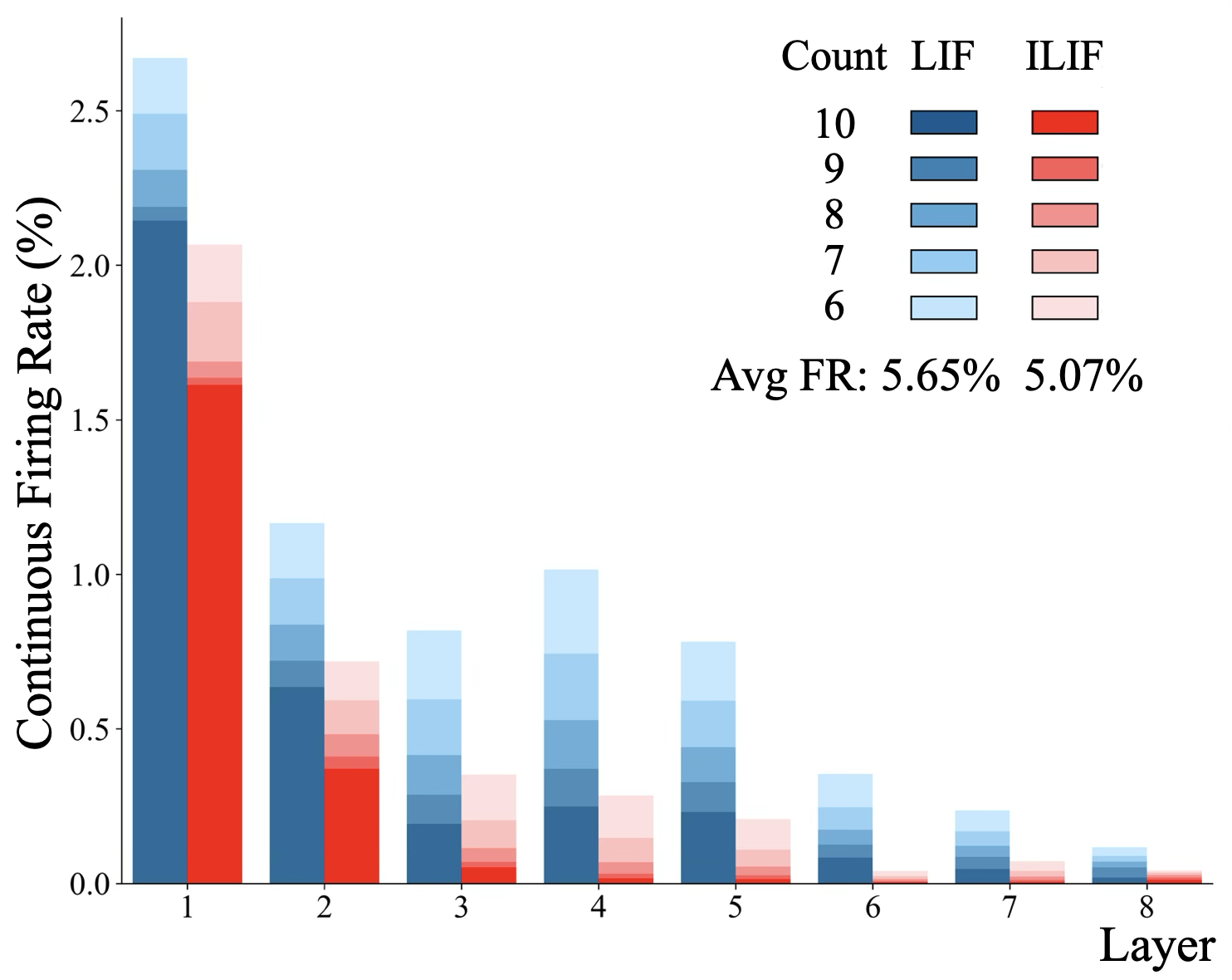}
        \vspace{-16pt}
        \caption{DVSCIFAR10}
        \label{fig:continous_10}
    \end{subfigure}
    \hfill
    \begin{subfigure}[b]{0.23\textwidth}
        \centering
        \includegraphics[width=\textwidth]{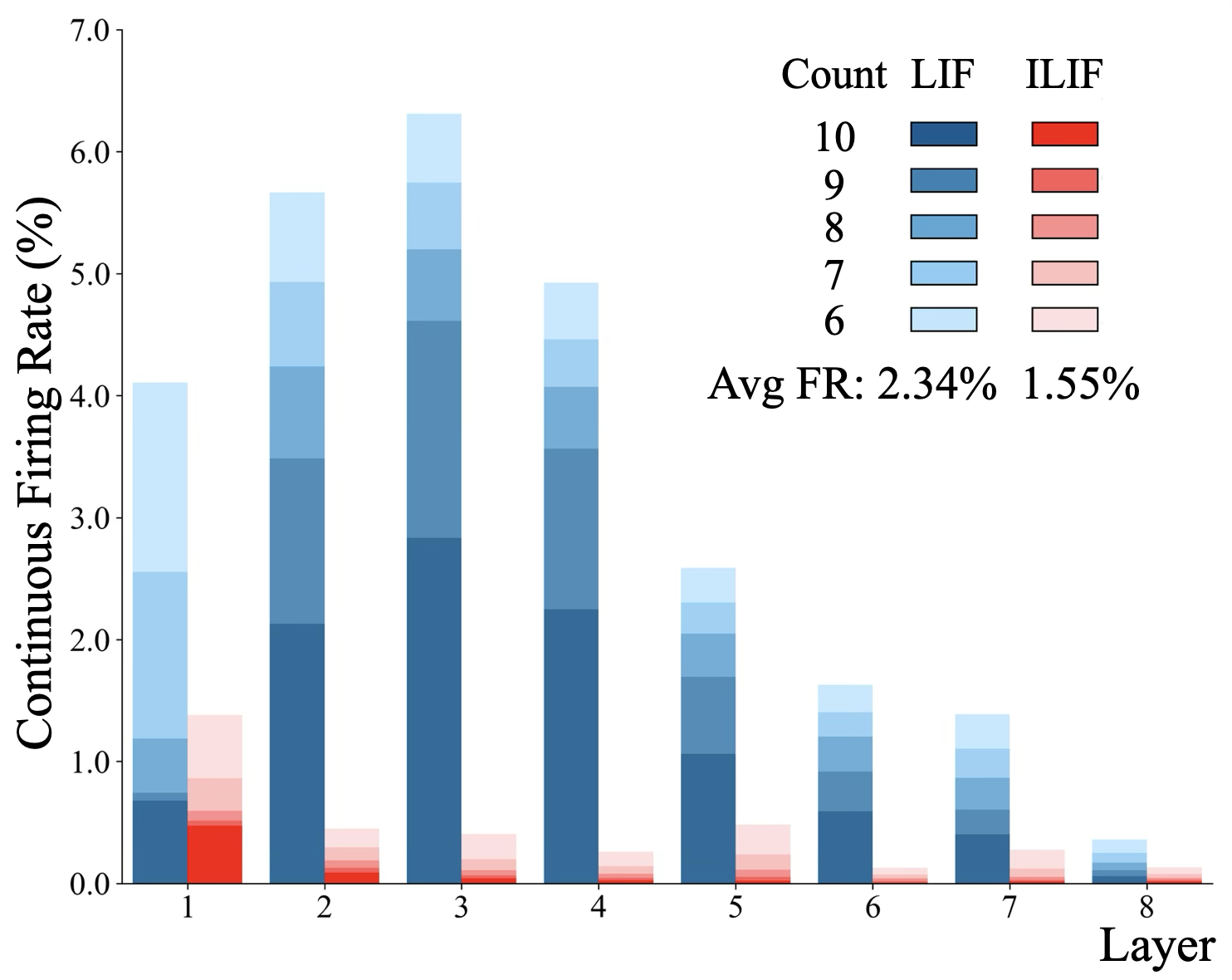}
        \vspace{-16pt}
        \caption{DVSGesture}
        \label{fig:continous_dvs}
    \end{subfigure}
    \caption{Continuous firing rate comparison on neuromorphic data.}
    \label{fig:continous_overhalf}
\end{figure}

\subsection{Evaluation of the Effect of \(\gamma\) on ILIF}
Figure~\ref{fig:gamma_lif_ilif} illustrates the impact of varying \(\gamma\) on the firing rate and accuracy of the LIF and ILIF models. As \(\gamma\) increases, the LIF model experiences a significant decline in accuracy due to heightened activation, accompanied by a relatively higher firing rate. In contrast, the ILIF model consistently demonstrates a lower firing rate, highlighting its inhibitory effect. Furthermore, the ILIF model works synergistically with the backpropagation shortcut provided by its inhibitory units, ensuring more stable accuracy across all \(\gamma\) values.
\setcounter{figure}{4}
\begin{figure}[h]
    \centering
    \includegraphics[width=0.95\linewidth]{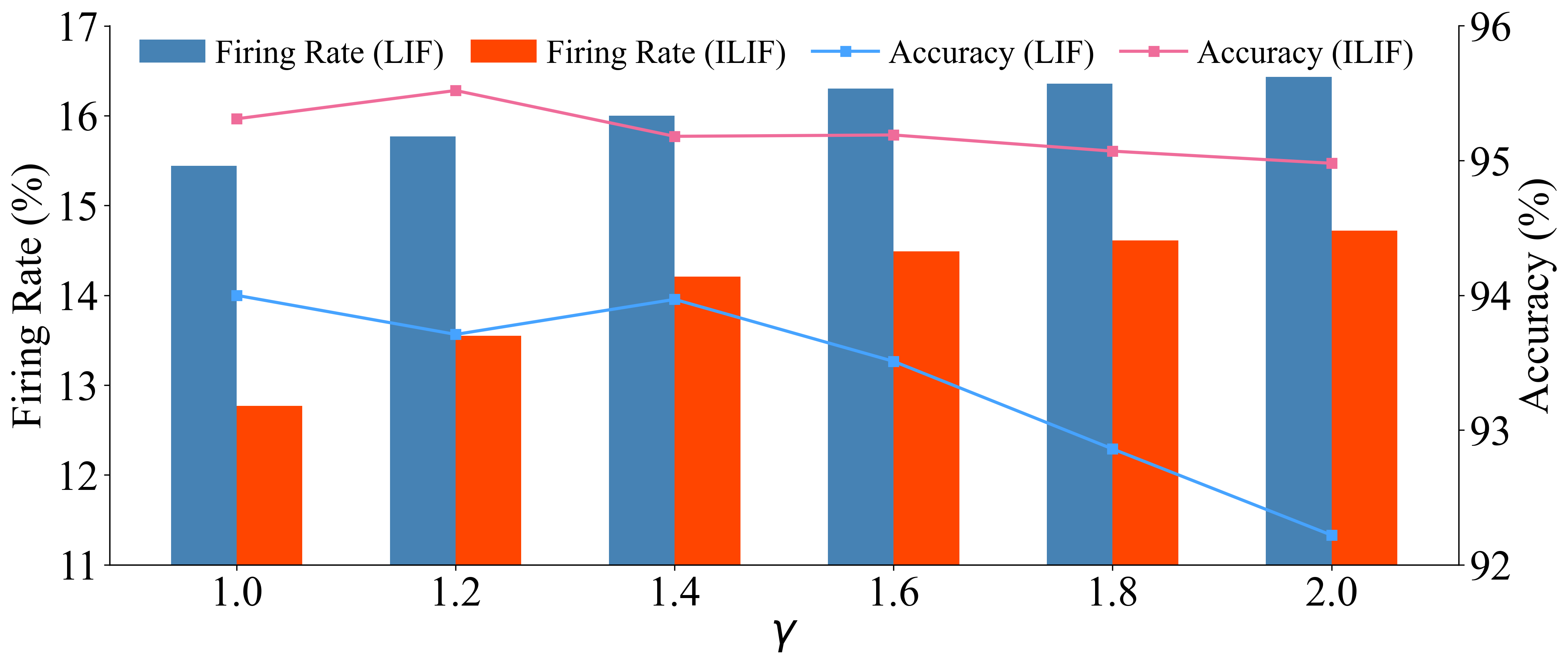}
    \captionsetup{skip=2pt}
    \caption{Firing rate and accuracy comparison for LIF and ILIF with respect to \(\gamma\) on CIFAR10}
    \label{fig:gamma_lif_ilif}
\end{figure}

\subsection{Ablation Study}
We performed ablation experiments to evaluate the contributions of each inhibitory unit in the ILIF model. As shown in Table~\ref{tb:Ablation_Study_ILIF}, removing either unit results in a noticeable drop in accuracy, confirming their necessity. To explore the generalizability of our proposed inhibitory mechanism, we integrated it into the PLIF model \cite{rw:plif}, a LIF variation that adjusts decay parameters. This integration (Table~\ref{tb:Impact_Inhibitory_Unit}) significantly enhances performance, demonstrating the value of our inhibitory mechanism given that PLIF only modifies decay parameters and inherits the limitations of LIF neurons.

\begin{table}[!htb]
\centering
\setlength{\tabcolsep}{2pt} 
\caption{Ablation study of ILIF model components.}
\renewcommand{\arraystretch}{0.9} 
\begin{tabular}{c c|c c c c}
\toprule
\multicolumn{2}{c|}{Settings} & \multirow{2}{*}{CIFAR10} & \multirow{2}{*}{CIFAR100} & \multirow{2}{*}{DVSCIFAR10} & \multirow{2}{*}{DVSGesture} \\ \cline{1-2}
MPIU & CIU &  &  &  &  \\ \midrule
\checkmark & \checkmark & \textbf{95.49} & \textbf{78.51} & \textbf{78.60} & \textbf{97.92} \\
\checkmark & \xmark  & 95.11 & 78.06 & 78.21 & 97.56 \\
\xmark  & \checkmark & 94.63 & 76.84 & 77.13 & 97.22 \\
\xmark  & \xmark  & 93.76 & 75.46 & 75.50 & 96.88 \\ 
\bottomrule
\end{tabular}
\label{tb:Ablation_Study_ILIF}
\end{table}

\begin{table}[!htb]
\centering
\setlength{\tabcolsep}{2pt} 
\caption{Impact of adding inhibitory unit to PLIF model.}
\renewcommand{\arraystretch}{0.9} 
\begin{tabular}{ccccc}
\toprule
Method & CIFAR10 & CIFAR100 & DVSCIFAR10 & DVSGesture \\
\midrule
PLIF            & 94.25           & 75.65            & 75.32             & 95.49             \\ 
IPLIF           & 95.61           & 78.66            & 78.41             & 96.55             \\ 
\bottomrule
\end{tabular}
\label{tb:Impact_Inhibitory_Unit}
\end{table}

\section{Conclusion}
This paper analyzes the limitations of the LIF model, specifically overactivation and gradient vanishing caused by the SG support width \(\gamma\). Inspired by biological inhibitory mechanisms, we propose the ILIF model, which integrates MPIU and CIU to regulate neural activity and enhance gradient propagation simultaneously. Experiments show that ILIF achieves state-of-the-art performance, reduces continuous firing rates, and enhances result stability, providing an efficient and reliable solution for advancing neuromorphic computing.

\appendix
\onecolumn
\section{Gradient Derivation Based on LIF model}
\label{app:LIF_backprop}
With surrogate gradient method, the gradients are computed by BPTT, meaning backpropagating the loss $\mathcal{L}$ through all time steps. The gradients at \(l\)-th are calculates as:
\begin{equation}
\frac{\partial \mathcal{L}}{\partial \bm{W}^l} = \sum_{t=1}^T \frac{\partial \mathcal{L}}{\partial \bm{U}^l[t]} \cdot \frac{\partial \bm{U}^l[t]}{\partial \bm{W}^l}, l = L, L - 1, \cdots, 1,
\label{eq:derivative}
\end{equation}
According to the chain rule, when \( t = T\), we have
\begin{equation}
\frac{\partial \mathcal{L}}{\partial \bm{U}^{l}[T]} = \frac{\partial \mathcal{L}}{\partial \bm{S}^l[T]} \cdot \frac{\partial \bm{S}^l[T]}{\partial \bm{U}^{l}[T]},
\end{equation}
When \( t = T - 1, \ldots, 1\), this gradient can be decomposed into spatial and temporal components:
\begin{equation}
\frac{\partial \mathcal{L}}{\partial \bm{U}^l[t]} = 
\underbrace{\frac{\partial \mathcal{L}}{\partial \bm{S}^l[t]} \cdot \frac{\partial \bm{S}^l[t]}{\partial \bm{U}^l[t]}}_{\text{Spatial Term}} + 
\underbrace{\frac{\partial \mathcal{L}}{\partial \bm{U}^l[t+1]} \cdot \frac{\partial \bm{U}^l[t+1]}{\partial \bm{U}^l[t]}}_{\text{Temporal Term}},
\label{eq:gradient_decompose}
\end{equation}
And it can be further derived recursively as follows:
\begin{equation}
\begin{aligned}
\frac{\partial \mathcal{L}}{\partial \bm{U}^{l}[t]} &= \frac{\partial \mathcal{L}}{\partial \bm{S}^l[t]} \frac{\partial \bm{S}^l[t]}{\partial \bm{U}^{l}[t]} + \frac{\partial \mathcal{L}}{\partial \bm{U}^{l}[t+1]} 
\overset{\text{decay rate } \lambda}{\overbrace{
\frac{\partial \bm{U}^{l}[t+1]}{\partial \bm{m}^{l}[t]} }}
\overset{\epsilon^l[t]}{\overbrace{(\frac{\partial \bm{m}^{l}[t]}{\partial \bm{U}^{l}[t]} + \frac{\partial \bm{m}^{l}[t]}{\partial \bm{S}^l[t]} \frac{\partial \bm{S}^l[t]}{\partial \bm{U}^{l}[t]})}}
\\
&= \frac{\partial \mathcal{L}}{\partial \bm{S}^l[t]} \frac{\partial \bm{S}^l[t]}{\partial \bm{U}^{l}[t]} + \left( \frac{\partial \mathcal{L}}{\partial \bm{S}^l[t+1]} \frac{\partial \bm{S}^l[t+1]}{\partial \bm{U}^{l}[t+1]} + \frac{\partial \mathcal{L}}{\partial \bm{U}^{l}[t+2]} \lambda \epsilon^l[t+1] \right) \lambda \epsilon^l[t] \\
&= \frac{\partial \mathcal{L}}{\partial \bm{S}^l[t]} \cdot \frac{\partial \bm{S}^l[t]}{\partial \bm{U}^{l}[t]} + \frac{\partial \mathcal{L}}{\partial \bm{S}^l[t+1]} \cdot \frac{\partial \bm{S}^l[t+1]}{\partial \bm{U}^{l}[t+1]} \cdot \lambda \epsilon^l[t] + \frac{\partial \mathcal{L}}{\partial \bm{U}^{l}[t+2]} \cdot \lambda \epsilon^l[t+1] \cdot \lambda \epsilon^l[t] \\
&= \dots \\
&= \frac{\partial \mathcal{L}}{\partial \bm{S}^l[t]} \cdot \frac{\partial \bm{S}^l[t]}{\partial \bm{U}^{l}[t]} + \sum_{t' = t+1}^{T} \left( \frac{\partial \mathcal{L}}{\partial \bm{S}^l[t']} \cdot \frac{\partial \bm{S}^l[t']}{\partial \bm{U}^{l}[t']} \cdot \prod_{t'' = t+1}^{t'} \lambda \epsilon^l[t''-1] \right ) \\
&= \sum_{t' = t}^{T} \left( \frac{\partial \mathcal{L}}{\partial \bm{S}^l[t']} \cdot \frac{\partial \bm{S}^l[t']}{\partial \bm{U}^l[t]} \cdot \prod_{t'' = t+1}^{t'} \lambda \epsilon^l[t''-1] \right),
\label{lif_gradient_analysis_in_appendix}
\end{aligned}
\end{equation}
The gradient \(\frac{\partial \mathcal{L}}{\partial \bm{S}^l[t]} \) varies depending on the layer, and can be expressed as:
\begin{equation}
    \frac{\partial\mathcal{L}}{\partial\bm{S}^l[t]} =
    \begin{cases} 
        \frac{\partial\mathcal{L}}{\partial\bm{S}^L[t]}, & \text{if } l = \mathcal{L}, \\
        \frac{\partial\mathcal{L}}{\partial\bm{U}^{l+1}[t]} \frac{\partial\bm{U}^{l+1}[t]}{\partial\bm{S}^l[t]}, & \text{if } l = \mathcal{L}-1, \dots, 1.
    \end{cases}
\label{eq:partialL/partialS_app}
\end{equation}

\section{Analysis of \(\gamma\)'s Impact on Temporal Gradient Vanishing}
\label{app:gamma_temporal_impact}
Here, we analyze the impact of \(\gamma\) on temporal gradient vanishing to complete the proof of Lemma \ref{lemma_app: gradient vanishing and gamma} according to the analysis in \cite{huang2024clif}:

\setcounter{lemma}{1}
\begin{lemma}
    \label{lemma_app: gradient vanishing and gamma}
    The likelihood of experiencing gradient vanishing is inversely correlated with \(\gamma\).
\end{lemma}

\begin{proof}

From Equation~\,\eqref{lif_gradient_analysis_in_appendix}, the gradient backpropagated from time \(t'\) to an earlier time \(t\) can be written as
\begin{equation}
    \frac{\partial \mathcal{L}}{\partial \bm{S}^l[t']} 
    \;\cdot\; \frac{\partial \bm{S}^l[t']}{\partial \bm{U}^{l}[t']} 
    \;\cdot\; \prod_{t''=t+1}^{t'} \lambda\,\epsilon^l[t''-1],
\end{equation}
where
\begin{equation}
    \epsilon^l[t] = 1 - V_{\text{th}}\,H'(x) \;=\;
    \begin{cases}
        1 - \frac{V_{\text{th}}}{\gamma}, & \text{if } \bigl|U^l[t] - V_{\text{th}}\bigr|\le \tfrac{\gamma}{2}, \\[6pt]
        0, & \text{if } \bigl|U^l[t] - V_{\text{th}}\bigr|> \tfrac{\gamma}{2}.
    \end{cases}
\end{equation}

The key observation is that \(\epsilon^l[t]\) can become very small or even zero, depending on \(\gamma\). To avoid spatial gradient explosion, prior work \cite{wu2019direct} often sets \(\gamma \ge V_{\text{th}}\). In this regime (\(\gamma > V_{\text{th}}\)), we have
\begin{equation}
    0 < \frac{V_{\text{th}}}{\gamma} < 1 
    \quad\Longrightarrow\quad
    0 < 1 - \frac{V_{\text{th}}}{\gamma} < 1.
\end{equation}
Hence as \(\gamma\) decreases from values \(\gamma > V_{\text{th}}\) down toward \(\gamma = V_{\text{th}}\), the quantity 
\(\bigl(1 - \tfrac{V_{\text{th}}}{\gamma}\bigr)\) shrinks toward zero. Because this factor appears repeatedly in the product \(\prod_{t''=t+1}^{t'} \epsilon^l[t''-1]\), it accelerates the convergence of gradients to zero (i.e., gradient vanishing) whenever a neuron is near threshold.

Furthermore, if one adopts the often-used default \(\gamma = V_{\text{th}} = 1\) \cite{comparison-TET}, then
\begin{equation}
    1 - \frac{V_{\text{th}}}{\gamma} = 1 - 1 = 0,
\end{equation}
which yields a complete cutoff. Specifically, whenever \(\bigl|U^l[t] - V_{\text{th}}\bigr| \le \tfrac{\gamma}{2}\), we have \(\epsilon^l[t] = 0\), causing all past gradients from times \(t' \le t\) to be nullified by the product.

Therefore, as \(\gamma\) approaches \(V_{\text{th}}\), the factor \(1 - \tfrac{V_{\text{th}}}{\gamma}\) becomes vanishingly small. This markedly increases the likelihood of gradient vanishing for near-threshold events—demonstrating that \textbf{reducing} the surrogate gradient support width \(\gamma\) \textbf{exacerbates} vanishing of the temporal gradient, especially at the critical point \(\gamma = V_{\text{th}}\). This completes the proof of Lemma \ref{lemma_app: gradient vanishing and gamma}.
\end{proof}

\section{Gradient Derivation Based on ILIF model}
\label{app:ILIF_derivation}
According to the chain rule and the propagation process illustrated in Figure~\ref{fig:neuron_subfig4} and \ref{fig:neuron_subfig5}, we can derive the following expression:

\begin{equation}
    \frac{\partial \mathcal{L}}{\partial \bm{U}^{l}[t]} = 
    \underbrace{\frac{\partial \mathcal{L}}{\partial \bm{S}^l[t]} \frac{\partial \bm{S}^l[t]}{\partial \bm{U}^{l}[t]}}_{\text{Spatial part}}
    + \underbrace{\frac{\partial \mathcal{L}}{\partial \bm{U}^{l}[t+1]} 
    \overset{\lambda}{\overbrace{
    \frac{\partial \bm{U}^{l}[t+1]}{\partial \bm{m}^{l}[t]} }}
    \overset{1}{\overbrace{
    \frac{\partial \bm{m}^{l}[t]}{\partial \bar{\bm{m}}^{l}[t]} }}
    \epsilon^l[t]}_{\text{Temporal part}} + 
    \overset{\text{define as } \phi^l[t]}{\overbrace{\underbrace{\frac{\partial \mathcal{L}}{\partial \mathbb{U}^{l}[t]}
    \xi^{l}[t] +
    \frac{\partial \mathcal{L}}{\partial \mathbb{I}^{l}[t]}
    \delta^{l}[t]}_{\text{Inhibitory additional part}}}}
    \label{eq:Gradient derivation of ILIF model}
\end{equation}

In this equation,
\begin{equation}
    \epsilon^l[t] \equiv 
    \frac{\partial \bar{\bm{m}}^{l}[t]}{\partial \bm{U}^l[t]} + \frac{\partial \bar{\bm{m}}^{l}[t]}{\partial \bm{S}^l[t]} \frac{\partial \bm{S}^l[t]}{\partial \bm{U}^{l}[t]}  = 1 - \frac{V_{\text{th}}}{\gamma}\geq 0
\end{equation}

\begin{equation}
    \xi^{l}[t] \equiv 
    \underbrace{\frac{\partial \mathbb{U}^l[t]}{\partial \bar{\bm{m}}^{l}[t]}}_{\geq 0}
    \underbrace{\epsilon^l[t]}_{\geq 0} +
    \underbrace{\frac{\partial \mathbb{U}^l[t]}{\partial \bm{S}^l[t]}}_{>0}
    \underbrace{\frac{\partial \bm{S}^l[t]}{\partial \bm{U}^l[t]}}_{>0} >0
\end{equation}

\begin{equation}
    \delta^{l}[t] \equiv 
    \underbrace{\frac{\partial \mathbb{I}^l[t]}{\partial \bm{S}^l[t]}}_{>0}
    \underbrace{\frac{\partial \bm{S}^l[t]}{\partial \bm{U}^l[t]}}_{>0} >0
\end{equation}

Next, we recursively expand the equation in a concise manner and examine the intermediate components in Appendix \ref{app:ILIF_gradient_analysis}.
\begin{equation}
\begin{aligned}
\frac{\partial \mathcal{L}}{\partial \bm{U}^{l}[t]} &= \frac{\partial \mathcal{L}}{\partial \bm{S}^l[t]} \frac{\partial \bm{S}^l[t]}{\partial \bm{U}^{l}[t]} + \frac{\partial \mathcal{L}}{\partial \bm{U}^{l}[t+1]} \lambda \epsilon^l[t] + \phi^l[t] \\
&= \frac{\partial \mathcal{L}}{\partial \bm{S}^l[t]} \frac{\partial \bm{S}^l[t]}{\partial \bm{U}^{l}[t]} + \left( \frac{\partial \mathcal{L}}{\partial \bm{S}^l[t+1]} \frac{\partial \bm{S}^l[t+1]}{\partial \bm{U}^{l}[t+1]} + \frac{\partial \mathcal{L}}{\partial \bm{U}^{l}[t+2]} \lambda \epsilon^l[t+1] + \phi^l[t+1] \right) \lambda \epsilon^l[t] + \phi^l[t] \\
&= \frac{\partial \mathcal{L}}{\partial \bm{S}^l[t]} \frac{\partial \bm{S}^l[t]}{\partial \bm{U}^{l}[t]} + \frac{\partial \mathcal{L}}{\partial \bm{S}^l[t+1]} \frac{\partial \bm{S}^l[t+1]}{\partial \bm{U}^{l}[t+1]} \lambda \epsilon^l[t] + \frac{\partial \mathcal{L}}{\partial \bm{U}^{l}[t+2]} \lambda \epsilon^l[t+1] \lambda \epsilon^l[t] + \phi^l[t+1] \lambda \epsilon^l[t] + \phi^l[t] \\
&= \frac{\partial \mathcal{L}}{\partial \bm{S}^l[t]} \frac{\partial \bm{S}^l[t]}{\partial \bm{U}^{l}[t]} + \frac{\partial \mathcal{L}}{\partial \bm{S}^l[t+1]} \frac{\partial \bm{S}^l[t+1]}{\partial \bm{U}^{l}[t+1]} \lambda \epsilon^l[t] + \frac{\partial \mathcal{L}}{\partial \bm{S}^l[t+2]} \frac{\partial \bm{S}^l[t+2]}{\partial \bm{U}^{l}[t+2]} \lambda \epsilon^l[t+1] \lambda \epsilon^l[t] \\
&\quad \quad+ \phi^l[t+2] \lambda \epsilon^l[t+1] \lambda \epsilon^l[t] + \phi^l[t+1] \lambda \epsilon^l[t] + \phi^l[t] \\
&= \dots \\
&= \textcolor{cyan!100}{
\frac{\partial \mathcal{L}}{\partial \bm{S}^l[t]} \frac{\partial \bm{S}^l[t]}{\partial \bm{U}^{l}[t]} + 
\sum_{t' = t+1}^{T} \left( \frac{\partial \mathcal{L}}{\partial \bm{S}^l[t']} \frac{\partial \bm{S}^l[t']}{\partial \bm{U}^{l}[t']} \prod_{t'' = t+1}^{t'} \lambda \epsilon^l[t''-1] \right)
} \textcolor{red!100}{ +
\sum_{t' = t+1}^{T} \left( \phi^l[t'] \prod_{t'' = t+1}^{t'} \lambda \epsilon^l[t''-1] \right)
+ \phi^l[t]
}
\end{aligned}
\end{equation}
where the blue part is the same as the computation in vanilla LIF model while the red part is the additional gradient introduced by the inhibitory units.

\section{Analysis of Additional Gradient Propagation of ILIF}
\label{app:ILIF_gradient_analysis}
The ILIF model introduces an inhibitory mechanism that influences gradient propagation through time. To understand its effect, we analyze the gradient contributions across different time steps.

When \(t=T\), the additional gradient is zero because the inhibition at the final time step occurs after the spike is generated, thus having no effect on the output.

When \(t=T-1\), the gradient contribution involves both the membrane potential and the input current, which can be expressed as:
\begin{equation}
    \frac{\partial \mathcal{L}}{\partial \mathbb{U}^{l}[T-1]} =
    \frac{\partial \mathcal{L}}{\partial \bm{S}^{l}[T]}
    \frac{\partial \bm{S}^{l}[T]}{\partial \bm{U}^{l}[T]}
    \frac{\partial \bm{U}^{l}[T]}{\partial \bm{m}^{l}[T-1]}
    \frac{\partial \bm{m}^{l}[T-1]}{\partial \mathbb{U}^{l}[T-1]}
\label{eq:membrane_t-1}
\end{equation}
and
\begin{equation}
    \frac{\partial \mathcal{L}}{\partial \mathbb{I}^{l}[T-1]} =
    \frac{\partial \mathcal{L}}{\partial \bm{S}^{l}[T]}
    \frac{\partial \bm{S}^{l}[T]}{\partial \bm{U}^{l}[T]}
    \frac{\partial \bm{U}^{l}[T]}{\partial \bm{I}^{l}[T]}
    \frac{\partial \bm{I}^{l}[T]}{\partial \mathbb{I}^{l}[T-1]}
\end{equation}

When \( t = T - 2, \ldots, 1\), we recursively expand the additional gradient parts. For the membrane potential, this expansion is written as:

\begin{equation}
    \frac{\partial \mathcal{L}}{\partial \mathbb{U}^{l}[t]} =
    \frac{\partial \mathcal{L}}{\partial \bm{S}^{l}[t+1]}
    \frac{\partial \bm{S}^{l}[t+1]}{\partial \bm{U}^{l}[t+1]}
    \frac{\partial \bm{U}^{l}[t+1]}{\partial \bm{m}^{l}[t]}
    \frac{\partial \bm{m}^{l}[t]}{\partial \mathbb{U}^{l}[t]} +
    \frac{\partial \mathcal{L}}{\partial \mathbb{U}^{l}[t+1]}
    \left( 
    \frac{\partial \mathbb{U}^{l}[t+1]}{\partial \mathbb{U}^{l}[t]} +
    \xi^{l}[t+1]
    \frac{\partial \bm{U}^{l}[t+1]}{\partial \bm{m}^{l}[t]}
    \frac{\partial \bm{m}^{l}[t]}{\partial \mathbb{U}^{l}[t]}
    \right)
\end{equation}

We further expand:
\begin{equation}
\begin{aligned}
    \frac{\partial \mathcal{L}}{\partial \mathbb{U}^{l}[t]} &= 
    \frac{\partial \mathcal{L}}{\partial \bm{S}^{l}[t+1]}
    \frac{\partial \bm{S}^{l}[t+1]}{\partial \bm{U}^{l}[t+1]} 
    \frac{\partial \bm{U}^{l}[t+1]}{\partial \bm{m}^{l}[t]} 
    \frac{\partial \bm{m}^{l}[t]}{\partial \mathbb{U}^{l}[t]} 
    + \frac{\partial \mathcal{L}}{\partial \mathbb{U}^{l}[t+1]} 
    \overset{\text{define as } \theta^l[t]}{\overbrace{\left( 
    \frac{\partial \mathbb{U}^{l}[t+1]}{\partial \mathbb{U}^{l}[t]} + 
    \xi^{l}[t+1] 
    \frac{\partial \bm{U}^{l}[t+1]}{\partial \bm{m}^{l}[t]} 
    \frac{\partial \bm{m}^{l}[t]}{\partial \mathbb{U}^{l}[t]} 
    \right)}} \\ 
    &= 
    \frac{\partial \mathcal{L}}{\partial \bm{S}^{l}[t+1]}
    \frac{\partial \bm{S}^{l}[t+1]}{\partial \bm{U}^{l}[t+1]} 
    \frac{\partial \bm{U}^{l}[t+1]}{\partial \bm{m}^{l}[t]} 
    \frac{\partial \bm{m}^{l}[t]}{\partial \mathbb{U}^{l}[t]} 
    + \left( \frac{\partial \mathcal{L}}{\partial \bm{S}^{l}[t+2]}
    \frac{\partial \bm{S}^{l}[t+2]}{\partial \bm{U}^{l}[t+2]} 
    \frac{\partial \bm{U}^{l}[t+2]}{\partial \bm{m}^{l}[t+1]} 
    \frac{\partial \bm{m}^{l}[t+1]}{\partial \mathbb{U}^{l}[t+1]} + 
    \frac{\partial \mathcal{L}}{\partial \mathbb{U}^{l}[t+2]} 
    \theta^l[t+1] \right) \theta^l[t] \\ 
    &= \frac{\partial \mathcal{L}}{\partial \bm{S}^{l}[t+1]}
    \frac{\partial \bm{S}^{l}[t+1]}{\partial \bm{U}^{l}[t+1]} 
    \frac{\partial \bm{U}^{l}[t+1]}{\partial \bm{m}^{l}[t]} 
    \frac{\partial \bm{m}^{l}[t]}{\partial \mathbb{U}^{l}[t]} + 
    \frac{\partial \mathcal{L}}{\partial \bm{S}^{l}[t+2]}
    \frac{\partial \bm{S}^{l}[t+2]}{\partial \bm{U}^{l}[t+2]} 
    \frac{\partial \bm{U}^{l}[t+2]}{\partial \bm{m}^{l}[t+1]} 
    \frac{\partial \bm{m}^{l}[t+1]}{\partial \mathbb{U}^{l}[t+1]} 
    \theta^l[t] \\ 
    &\quad + \frac{\partial \mathcal{L}}{\partial \bm{S}^{l}[t+3]}
    \frac{\partial \bm{S}^{l}[t+3]}{\partial \bm{U}^{l}[t+3]} 
    \frac{\partial \bm{U}^{l}[t+3]}{\partial \bm{m}^{l}[t+2]} 
    \frac{\partial \bm{m}^{l}[t+2]}{\partial \mathbb{U}^{l}[t+2]} 
    \theta^l[t+1] \theta^l[t] + \dots \\ 
    &= \frac{\partial \mathcal{L}}{\partial \bm{S}^{l}[t+1]}
    \frac{\partial \bm{S}^{l}[t+1]}{\partial \bm{U}^{l}[t+1]} 
    \frac{\partial \bm{U}^{l}[t+1]}{\partial \bm{m}^{l}[t]} 
    \frac{\partial \bm{m}^{l}[t]}{\partial \mathbb{U}^{l}[t]} 
    + \sum_{t' = t+1}^{T-2} 
    \frac{\partial \mathcal{L}}{\partial \bm{S}^{l}[t'+1]}
    \frac{\partial \bm{S}^{l}[t'+1]}{\partial \bm{U}^{l}[t'+1]} 
    \frac{\partial \bm{U}^{l}[t'+1]}{\partial \bm{m}^{l}[t']} 
    \frac{\partial \bm{m}^{l}[t']}{\partial \mathbb{U}^{l}[t']}
    \prod_{t'' = t+1}^{t'-1} \theta^l[t''] \\
    &= \sum_{t' = t}^{T-2} 
    \frac{\partial \mathcal{L}}{\partial \bm{S}^{l}[t'+1]}
    \frac{\partial \bm{S}^{l}[t'+1]}{\partial \bm{U}^{l}[t'+1]} 
    \frac{\partial \bm{U}^{l}[t'+1]}{\partial \bm{m}^{l}[t']} 
    \frac{\partial \bm{m}^{l}[t']}{\partial \mathbb{U}^{l}[t']}
    \prod_{t'' = t+1}^{t'-1} \theta^l[t'']
\label{eq:membrane_t-2}
\end{aligned}
\end{equation}

By combining Equation~\eqref{eq:membrane_t-1} and Equation~\eqref{eq:membrane_t-2}, we obtain the following expression for \( t = T - 1, \ldots, 1\)
\begin{equation}
\begin{aligned}
    \frac{\partial \mathcal{L}}{\partial \mathbb{U}^{l}[t]} &= 
    \frac{\partial \mathcal{L}}{\partial \bm{S}^{l}[T]}
    \frac{\partial \bm{S}^{l}[T]}{\partial \bm{U}^{l}[T]}
    \frac{\partial \bm{U}^{l}[T]}{\partial \bm{m}^{l}[T-1]}
    \frac{\partial \bm{m}^{l}[T-1]}{\partial \mathbb{U}^{l}[T-1]} +
    \sum_{t' = t}^{T-2} 
    \frac{\partial \mathcal{L}}{\partial \bm{S}^{l}[t'+1]}
    \frac{\partial \bm{S}^{l}[t'+1]}{\partial \bm{U}^{l}[t'+1]} 
    \frac{\partial \bm{U}^{l}[t'+1]}{\partial \bm{m}^{l}[t']} 
    \frac{\partial \bm{m}^{l}[t']}{\partial \mathbb{U}^{l}[t']}
    \prod_{t'' = t+1}^{t'-1} \theta^l[t''] \\
    &= \sum_{t' = t}^{T-1} 
\frac{\partial \mathcal{L}}{\partial \bm{S}^{l}[t'+1]}
\frac{\partial \bm{S}^{l}[t'+1]}{\partial \bm{U}^{l}[t'+1]} 
\frac{\partial \bm{U}^{l}[t'+1]}{\partial \bm{m}^{l}[t']} 
\frac{\partial \bm{m}^{l}[t']}{\partial \mathbb{U}^{l}[t']}
\prod_{t'' = t+1}^{t'-1} \theta^l[t'']
\end{aligned}
\end{equation}

To further analyze the effect on the gradient, we examine each component of the equation.

\begin{equation}
\begin{aligned}
    \frac{\partial \mathcal{L}}{\partial \mathbb{U}^{l}[t]} &= 
    \sum_{t' = t}^{T-1}
    \frac{\partial \mathcal{L}}{\partial \bm{S}^{l}[t]}
    \underbrace{\frac{\partial \bm{S}^{l}[t]}{\partial \bm{U}^{l}[t]}}_{>0}
    \underbrace{\frac{\partial \bm{U}^{l}[t'+1]}{\partial \bm{m}^{l}[t']}}_{>0}
    \underbrace{\frac{\partial \bm{m}^{l}[t']}{\partial \mathbb{U}^{l}[t']}}_{<0}
    \prod_{t'' = t+1}^{t'-1} \underbrace{\left( 
    \underbrace{\frac{\partial \mathbb{U}^{l}[t''+1]}{\partial \mathbb{U}^{l}[t'']}}_{=1} + 
    \underbrace{\xi^{l}[t''+1] 
    \frac{\partial \bm{U}^{l}[t''+1]}{\partial \bm{m}^{l}[t'']} 
    \frac{\partial \bm{m}^{l}[t'']}{\partial \mathbb{U}^{l}[t'']}}_{\in (-1,0)}
    \right)}_{>0}
\label{eq:membrane potential inhibition}
\end{aligned}
\end{equation}

According to the analysis in Lemma \ref{lemma: gamma and overactivation}, in the output layer, \(\frac{\partial \mathcal{L}}{\partial \bm{S}^L[t]} = \bm{Y}^L[t] - \hat{\bm{Y}}^L[t]\), where \(\bm{Y}^L[t]\) represents the average output (within [0,1]), and \(\hat{\bm{Y}}^L[t]\) represents the predicted values encoded in a one-hot format (taking values of either 0 or 1). Therefore,
\begin{equation}
    \frac{\partial \mathcal{L}}{\partial \mathbb{U}^{L}[t]} 
    \begin{cases}
        \geq 0, & \hat{\bm{Y}}[t'] = 1 \\ 
        \leq 0, & \hat{\bm{Y}}[t'] = 0 
    \end{cases}
\end{equation}

Similarly, the gradient associated with the input current can be expanded as:

\begin{equation}
    \frac{\partial \mathcal{L}}{\partial \mathbb{I}^{l}[t]} =
    \frac{\partial \mathcal{L}}{\partial \bm{S}^{l}[t+1]}
    \frac{\partial \bm{S}^{l}[t+1]}{\partial \bm{U}^{l}[t+1]}
    \frac{\partial \bm{U}^{l}[t+1]}{\partial \bm{I}^{l}[t+1]}
    \frac{\partial \bm{I}^{l}[t+1]}{\partial \mathbb{I}^{l}[t]} +
    \frac{\partial \mathcal{L}}{\partial \mathbb{I}^{l}[t+1]}
    \left( 
    \frac{\partial \mathbb{I}^{l}[t+1]}{\partial \mathbb{I}^{l}[t]} +
    \frac{\partial \mathbb{I}^{l}[t+1]}{\partial \bm{I}^{l}[t+1]}
    \frac{\partial \bm{I}^{l}[t+1]}{\partial \mathbb{I}^{l}[t]}
    \right)
\end{equation}

For the current additional derivative term, due to its rapid decay responsible for short-term adjustments and characterized by a very small coefficient, from \( t = T - 1, \ldots, 1\), the derivative can be approximated as:
\begin{equation}
    \frac{\partial \mathcal{L}}{\partial \mathbb{I}^{l}[t]} \approx
    \frac{\partial \mathcal{L}}{\partial \bm{S}^{l}[t+1]}
    \underbrace{\frac{\partial \bm{S}^{l}[t+1]}{\partial \bm{U}^{l}[t+1]}}_{>0}
    \underbrace{\frac{\partial \bm{U}^{l}[t+1]}{\partial \bm{I}^{l}[t+1]}}_{>0}
    \underbrace{\frac{\partial \bm{I}^{l}[t+1]}{\partial \mathbb{I}^{l}[t]}}_{<0}
\end{equation}

This follows the same sign pattern as the membrane potential derivative term:
\begin{equation}
    \frac{\partial \mathcal{L}}{\partial \mathbb{I}^{L}[t]} 
    \begin{cases}
        \geq 0, & \hat{\bm{Y}}[t'] = 1 \\ 
        \leq 0, & \hat{\bm{Y}}[t'] = 0 
    \end{cases}
\end{equation}

Similarly, based on Equation~\eqref{eq:partialL/partialS_app} and the chain rule, the sign of the additional gradient in the previous layers remains consistent with the weights, determined by the final output. Specifically, \(\frac{\partial \mathcal{L}}{\partial \mathbb{I}^{l}[t]}\) and \(\frac{\partial \mathcal{L}}{\partial \mathbb{U}^{l}[t]}\) are \(\geq 0\) if \(\hat{\bm{Y}}[t'] = 1\), and \(\leq 0\) if \(\hat{\bm{Y}}[t'] = 0\).

The gradient vanish problem is caused by the term \( \epsilon^l[t] \), which tends toward zero during calculations according to Lemma \ref{lemma: gamma and overactivation}. Here, we identify the additional gradient pathways in the expanded equation that facilitate gradient propagation during backpropagation. The additional gradient can be expressed as follows:

\begin{equation}
\begin{aligned}
    &\frac{\partial \mathcal{L}}{\partial \mathbb{U}^{l}[t]}
    \xi^{l}[t] +
    \frac{\partial \mathcal{L}}{\partial \mathbb{I}^{l}[t]}
    \delta^{l}[t] 
    = \frac{\partial \mathcal{L}}{\partial \mathbb{U}^{l}[t]}
    \left(
    \frac{\partial \mathbb{U}^l[t]}{\partial \bar{\bm{m}}^{l}[t]}
    \epsilon^l[t] +
    \frac{\partial \mathbb{U}^l[t]}{\partial \bm{S}^l[t]}
    \frac{\partial \bm{S}^l[t]}{\partial \bm{U}^l[t]}
    \right) +
    \frac{\partial \mathcal{L}}{\partial \mathbb{I}^{l}[t]}
    (\frac{\partial \mathbb{I}^l[t]}{\partial \bm{S}^l[t]}
    \frac{\partial \bm{S}^l[t]}{\partial \bm{U}^l[t]}) \\
    &=
    \frac{\partial \mathcal{L}}{\partial \mathbb{U}^{l}[t]}
    \cdot \lambda_{\mathbb{U}} \cdot \bm{S}^l[t] \cdot \epsilon^l[t] +
    \frac{\partial \mathcal{L}}{\partial \mathbb{U}^{l}[t]}
    \cdot \lambda_{\mathbb{U}} \cdot \bar{\bm{m}}^l[t] +
    \frac{\partial \mathcal{L}}{\partial \mathbb{I}^{l}[t]}
    \cdot \lambda_{\mathbb{I}} \cdot \bm{I}^l[t-1] 
    \quad\begin{cases}
        \geq 0, & \hat{\bm{Y}}[t'] = 1 \\ 
        \leq 0, & \hat{\bm{Y}}[t'] = 0 
    \end{cases}
\end{aligned}
\end{equation}

\section{Pseudocode of ILIF mdoel}
\label{app:pseudocode}
We provide the pseudocode of the ILIF model’s spiking process in Algorithm \ref{alg:pseudocode}. The red components represent the added inhibitory mechanisms compared to the vanilla LIF model. The abbreviations "P.L." and "P.T." denote the previous layer and previous time step, respectively.

\begin{algorithm}
\caption{Main Fire Procedure for LIF Model}
\label{alg:pseudocode}
\textbf{Input:} Total Time Steps $T$; Decay Rates $\lambda$, $\lambda_{\mathbb{U}}$, $\lambda_{\mathbb{I}}$; Threshold $V_{\text{th}}$ \\
\textbf{Initialize:} $\bm{U}_0 \gets 0$, $\mathbb{U}_0 \gets 0$, $\bm{I}_0 \gets 0$, $\mathbb{I}_0 \gets 0$

\setcounter{algorithm}{0}
\begin{algorithmic}[1]
\FOR{$t = 1$ to $T$}
    \STATE $\bm{I} \gets \bm{S}^{P.L.} \cdot W \textcolor{red!100}{- \mathbb{I}^{P.T.}}$ \hfill // \textit{Generate Input Current with Feedback Inhibition}
    \STATE $\bm{U} \gets \lambda \cdot \bm{m}^{P.T.} + \bm{I}$ \hfill // \textit{Update Membrane Potential}
    \STATE $\bm{S} \gets \mathbb{H}(\bm{U} - V_{\text{th}})$ \hfill // \textit{Check Threshold and Fire}
    \STATE $\bar{\bm{m}} \gets \bm{U} - \bm{S} \cdot V_{\text{th}}$ \hfill // \textit{Reset}
    \STATE $\textcolor{red!100}{\mathbb{U} \gets \lambda_{\mathbb{U}} \cdot \mathbb{U}^{P.T.} + \bm{S} \cdot \bar{\bm{m}}}$ \hfill // \textit{Update Membrane Potential Inhibition Unit}
    \STATE $\bm{m} \gets \bar{\bm{m}} \textcolor{red!100}{- \bm{S} \cdot \sigma(\mathbb{U})}$ \hfill // \textit{Inhibitory Reset}
    \STATE $\textcolor{red!100}{\mathbb{I} \gets \lambda_{\mathbb{I}} \cdot \mathbb{I}^{P.T.} + \bm{S} \cdot \bm{I}}$ \hfill // \textit{Update Current Inhibition Unit}
    \STATE \textbf{Output:} $\bm{S}$
\ENDFOR
\end{algorithmic}
\end{algorithm}

\section{Detailed Experimental Settings and Dateset Description}
\label{app:experiment_setting}
We utilize rectangle surrogate functions with \( \lambda = 1 \) and \(\tau = 1.1\), ensuring consistent random seed values of 1234. No data augmentation is applied across any of the datasets. All experiments are conducted on an NVIDIA RTX 4090D GPU. Detailed descriptions of the dataset configurations and default experimental setups are provided as shown in Table~\ref{tab:hyperparameters}.

\begin{table}[H]
\caption{Hyperparameter settings for various datasets.}
\centering
\begin{tabular}{ccccccccc}
\toprule
\textbf{Dataset} & \textbf{BS} & \textbf{Epochs} & \textbf{LR} & \textbf{Optimizer} & \textbf{Weight Decay} & \textbf{Dropout} & $\mathbb{U}$ & $\mathbb{I}$ \\
\midrule
CIFAR10      & 128 & 200 & 0.1  & SGD & 5e-05 & 0.1 & 1 & 0.03 \\
CIFAR100     & 128 & 200 & 0.1  & SGD & 5e-04 & 0.1 & 1 & 0.03 \\
DVSCIFAR10   & 128 & 200 & 0.05 & SGD & 5e-04 & 0.3 & 1 & 0.03 \\
DVSGesture   & 16  & 200 & 0.1  & SGD & 5e-04 & 0.4 & 1 & 0.05 \\
\bottomrule
\end{tabular}
\label{tab:hyperparameters}
\end{table}

\textbf{CIFAR10 and CIFAR100} are benchmark datasets for image classification. CIFAR10 consists of 10 classes with 60,000 32 \(\times\) 32 color images, divided into 50,000 training and 10,000 testing samples. CIFAR100 extends this to 100 classes grouped into 20 superclasses, with the same total number of images but fewer per class, making it more challenging. Direct encoding is utilized to convert image pixels into time series, with pixel values repeatedly fed into the input layer at each timestep.

\textbf{DVSCIFAR10} is an event-based dataset derived from CIFAR10 using a Dynamic Vision Sensor (DVS). It represents images as asynchronous event streams with spatial (x, y), polarity, and timestamp information. The dataset comprises 10,000 event-based images, where the pixel resolution has been expanded to 128×128. 

\textbf{DVSGesture} is also an event-based dataset comprising 11 hand gesture categories, performed by 29 participants under three different lighting conditions. It includes a total of 1,464 samples, with 1,176 samples in the training set and 288 samples in the test set. Each sample has a fixed resolution of 128x128 pixels.

\section{Continuous Firing Rate}
\label{app:continous_firing_rate}
This is a comparison of the continuous firing rates across different datasets and layers between the LIF and ILIF models. The ILIF model consistently exhibits lower continuous firing rates across almost all layers compared to the LIF model, demonstrating its inhibitory effect. Simultaneously, the ILIF model demonstrates lower average firing rates than the LIF model, with DVSGesture showing a significant reduction of over 30\%, as illustrated in Figure~\ref{fig:overall_continous_fr}.

\begin{figure}[htbp]
    \centering
    \begin{subfigure}{0.24\textwidth} 
        \includegraphics[width=\textwidth]{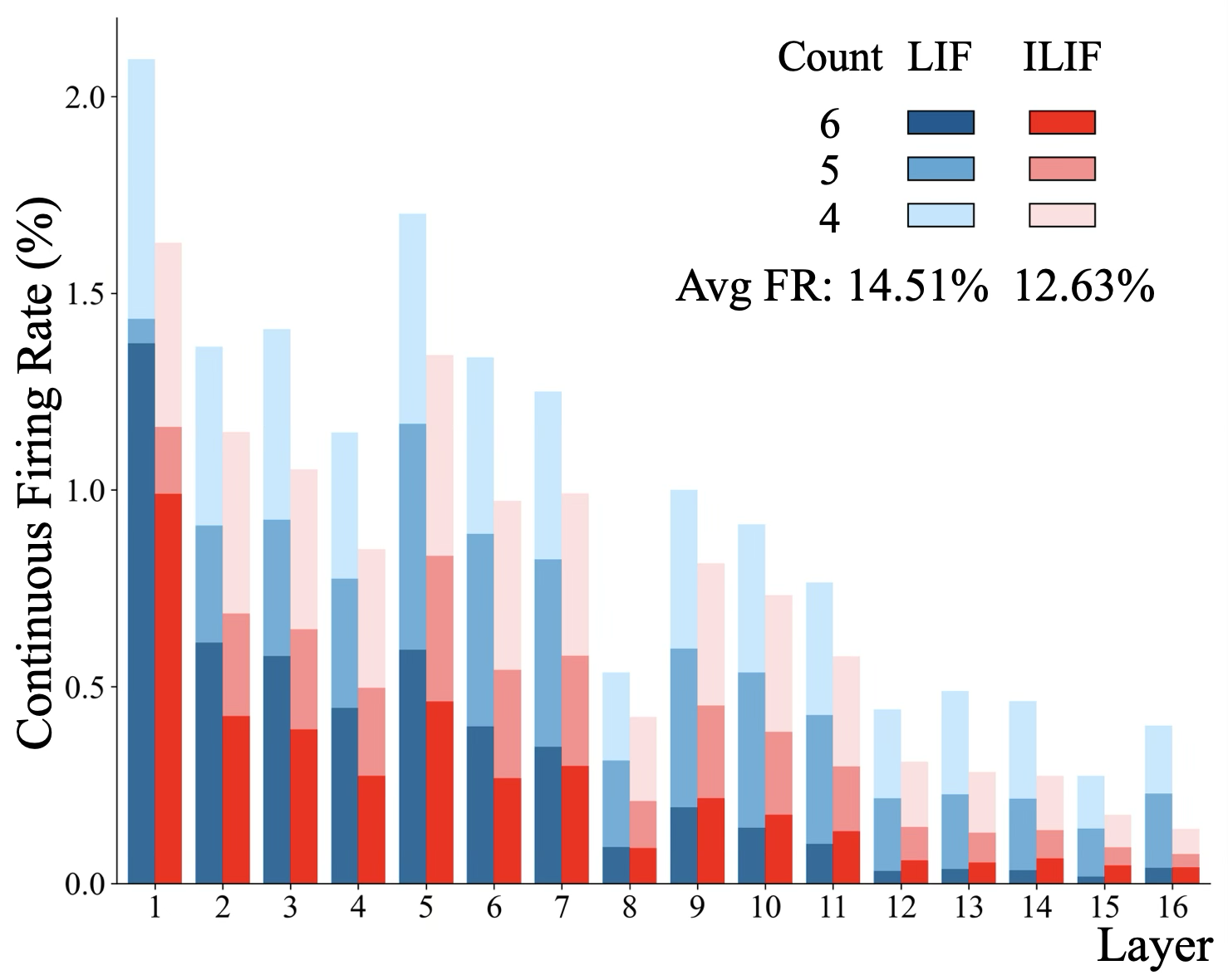} 
        \caption{CIFAR10} 
        \label{fig:sub1}
    \end{subfigure}
    \begin{subfigure}{0.24\textwidth}
        \includegraphics[width=\textwidth]{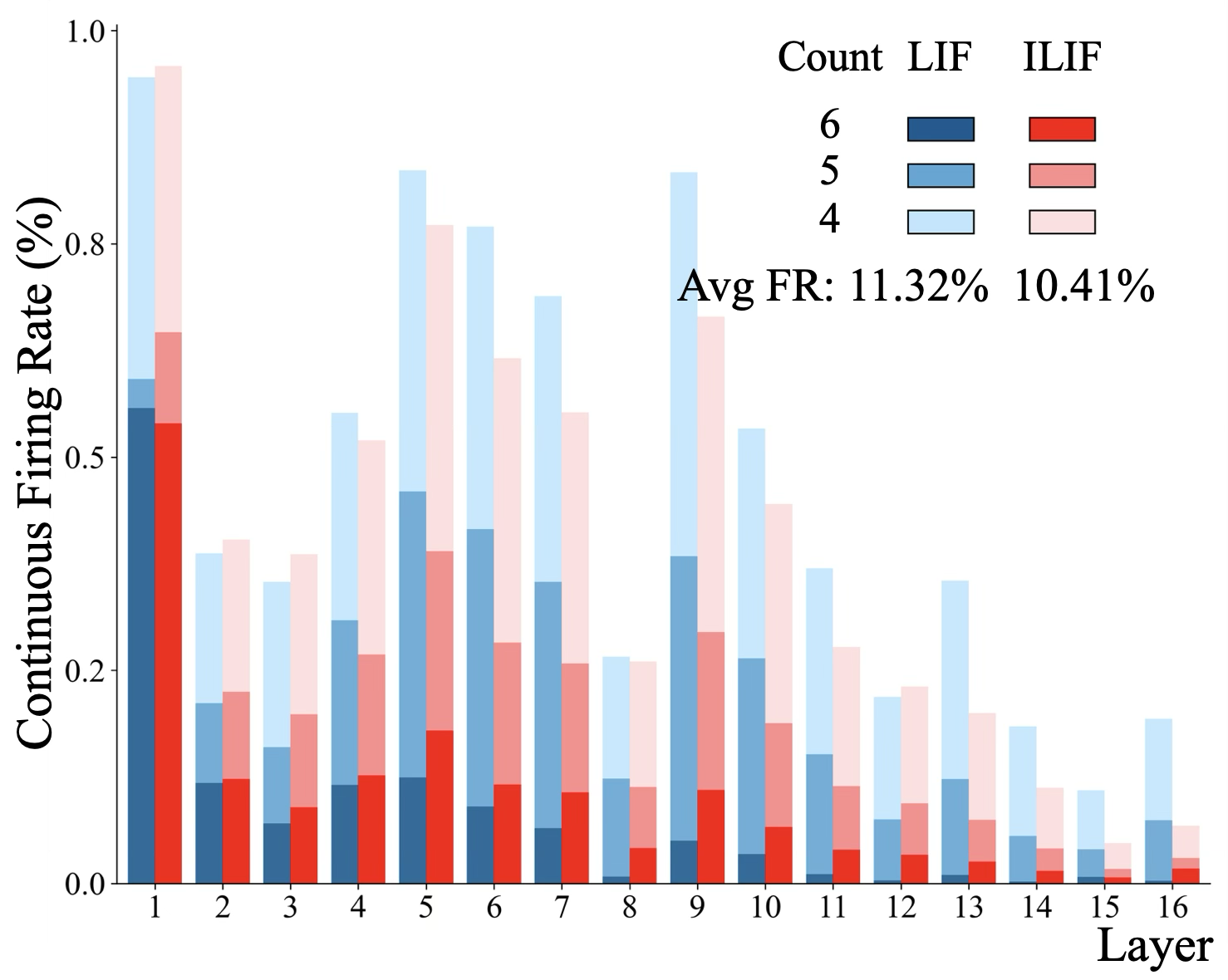}
        \caption{CIFAR100}
        \label{fig:sub2}
    \end{subfigure}
    \begin{subfigure}{0.24\textwidth}
        \includegraphics[width=\textwidth]{figures/continous_fire/continous_fire_dvscifar10.png}
        \caption{DVSCIFAR10}
        \label{fig:sub3}
    \end{subfigure}
    \begin{subfigure}{0.24\textwidth}
        \includegraphics[width=\textwidth]{figures/continous_fire/continous_fire_dvsges.png}
        \caption{DVSGesture}
        \label{fig:sub4}
    \end{subfigure}
    
    \caption{Continuous firing rate comparison across different datasets} 
    \label{fig:overall_continous_fr}
\end{figure}

\section{Energy Consumption}

To assess the energy efficiency of SNN, we adopt a standard approach from neuromorphic computing that estimates total synaptic operation power (SOP) based on the number of fundamental operations and their associated energy costs \cite{zhou2022spikformer}. Specifically, SOP is calculated as:

\begin{equation}
\text{SOP}_s = E_{\text{AC}} \cdot AC_s + E_{\text{MAC}} \cdot MAC_s
\end{equation}

where $E_{\text{AC}}$ and $E_{\text{MAC}}$ denote the energy consumption per accumulation (AC) and multiply-accumulate (MAC) operation, respectively. Following the energy model in \cite{han2015learning}, we assume each 32-bit floating-point addition consumes 0.9 picojoules (pJ), and each MAC operation consumes 4.6 pJ.

In SNN, neurons transmit binary spike signals, $s_i^l[t] \in \{0,1\}$, indicating whether neuron $i$ in layer $l$ fires at time step $t$. A spike activates all its outgoing synapses, each performing an addition. If a neuron has $f_i^l$ outgoing connections (fan-out), the total number of AC operations across the network is given by:

\begin{equation}
AC_s = \sum_{t=1}^{T} \sum_{l=1}^{L-1} \sum_{i=1}^{N^l} f_i^l \cdot s_i^l[t]
\end{equation}

Here, $T$ is the total number of time steps, $L$ is the number of layers, and $N^l$ is the number of neurons in layer $l$.

By contrast, ANN compute statically without temporal dynamics. Each neuron performs a single forward pass involving a fixed number of MAC operations, determined solely by its synaptic connections:

\begin{equation}
MAC_s = \sum_{l=1}^{L-1} \sum_{i=1}^{N^l} f_i^l
\end{equation}

Combining these operation counts with their energy costs yields the total SOP for any given network configuration.
As shown in Table~\ref{tab:neuron_energy_comparison}, SNNs demonstrate a clear advantage in energy efficiency over ANNs. Compared to LIF, ILIF incurs more MAC operations due to the inclusion of the inhibitory unit. However, its significantly lower spike rate not only offsets this overhead but leads to an overall reduction in energy consumption. On the DVS-Gesture dataset, ILIF shows slightly higher energy usage than LIF, attributed to the dataset’s higher resolution and extended temporal length. Despite this, ILIF consistently achieves better representational accuracy with fewer spikes and lower energy cost overall. These results highlight the effectiveness of incorporating inhibitory mechanism in achieving energy-efficient yet expressive spiking models.

\begin{table}[ht]
\caption{Comparison of Synaptic Operations and Energy Consumption Across ANN, LIF, and ILIF}
\centering
\begin{tabular}{c c c c c c c}
\toprule
\textbf{Dataset} & \textbf{Network Architecture} & \textbf{Method} & \textbf{T} & \textbf{ACs (M)} & \textbf{MACs (M)} & \textbf{SOP Energy ($\mu$J)} \\
\midrule
\multirow{3}{*}{CIFAR10} 
  & \multirow{3}{*}{ResNet18} & ANN  & 1 & 0 & 549.13 & 2525.980 \\
  &                            & LIF & 6 & 568.26 & 3.34 & 526.806 \\
  &                            & ILIF  & 6 & 488.12 & 6.68 & 470.056 \\
\midrule
\multirow{3}{*}{CIFAR100} 
  & \multirow{3}{*}{ResNet18} & ANN  & 1 & 0 & 549.13 & 2525.980 \\
  &                            & LIF & 6 & 438.70 & 3.34 & 410.203 \\
  &                            & ILIF  & 6 & 409.11 & 6.68 & 398.951 \\
\midrule
\multirow{2}{*}{DVSCIFAR10} 
  & \multirow{2}{*}{VGG11} & LIF & 10 & 201.58 & 3.41 & 197.106 \\
  &                        & ILIF  & 10 & 157.40 & 6.82 & 173.031 \\
\midrule
\multirow{2}{*}{DVSGesture} 
  & \multirow{2}{*}{VGG11} & LIF & 20 & 820.05 & 48.50 & 961.133 \\
  &                        & ILIF  & 20 & 653.86 & 96.99 & 1034.647 \\
\bottomrule
\end{tabular}
\label{tab:neuron_energy_comparison}
\end{table}

\section{Loss Curve}
As shown in Figure~\ref{fig:loss_comparison}, the loss curves of the vanilla LIF and ILIF models indicate that the vanilla LIF model undergoes fluctuations during training—mainly due to excessive neuron activation and gradient loss across time—thereby converging more slowly and ending at a higher final loss. By contrast, the ILIF model exhibits a smoother, steadily declining loss curve, as its inhibitory units provide shortcuts for rapid gradient backpropagation and help suppress activated neurons. This design ultimately enables the ILIF model to converge to a much lower loss value and attain higher accuracy.

\begin{figure}[h]
    \centering
    \begin{subfigure}[b]{0.24\textwidth}
        \centering
        \includegraphics[width=\textwidth]{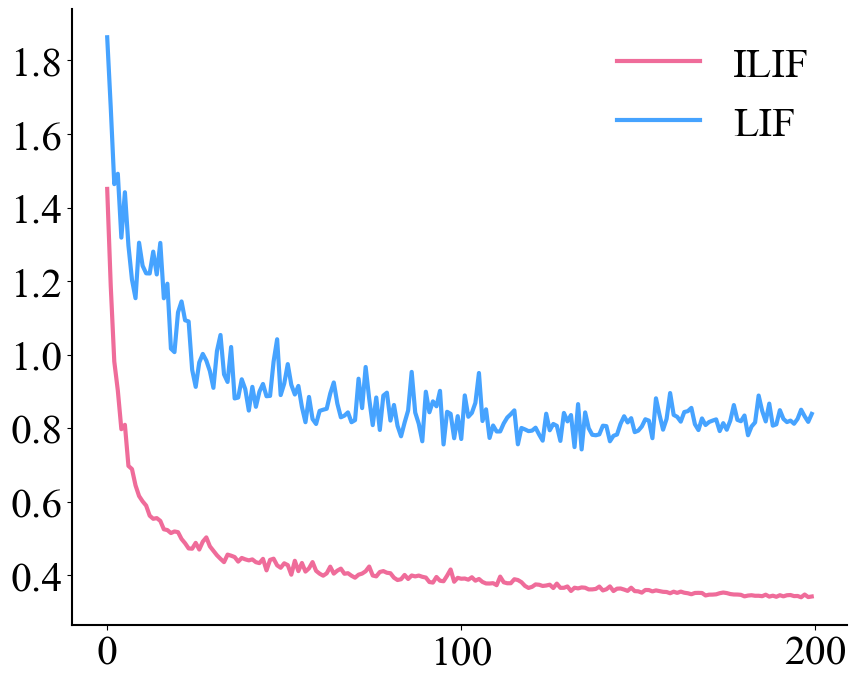}
        \vspace{-16pt}
        \caption{CIFAR10}
        \label{fig:cifar10_loss}
    \end{subfigure}
    \hfill
    \begin{subfigure}[b]{0.24\textwidth}
        \centering
        \includegraphics[width=\textwidth]{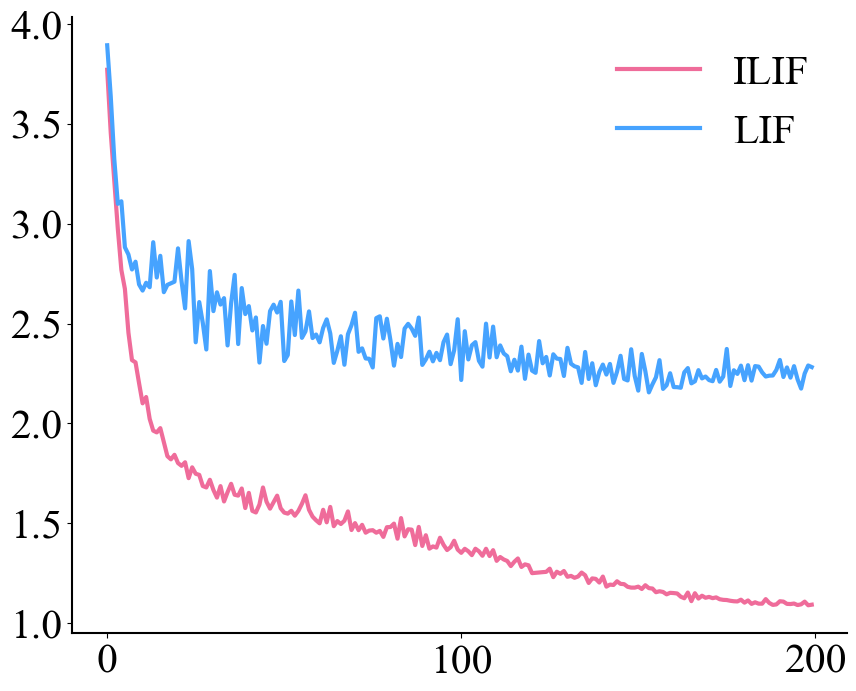}
        \vspace{-16pt}
        \caption{CIFAR100}
        \label{fig:cifar100_loss}
    \end{subfigure}
    \hfill
    \begin{subfigure}[b]{0.24\textwidth}
        \centering
        \includegraphics[width=\textwidth]{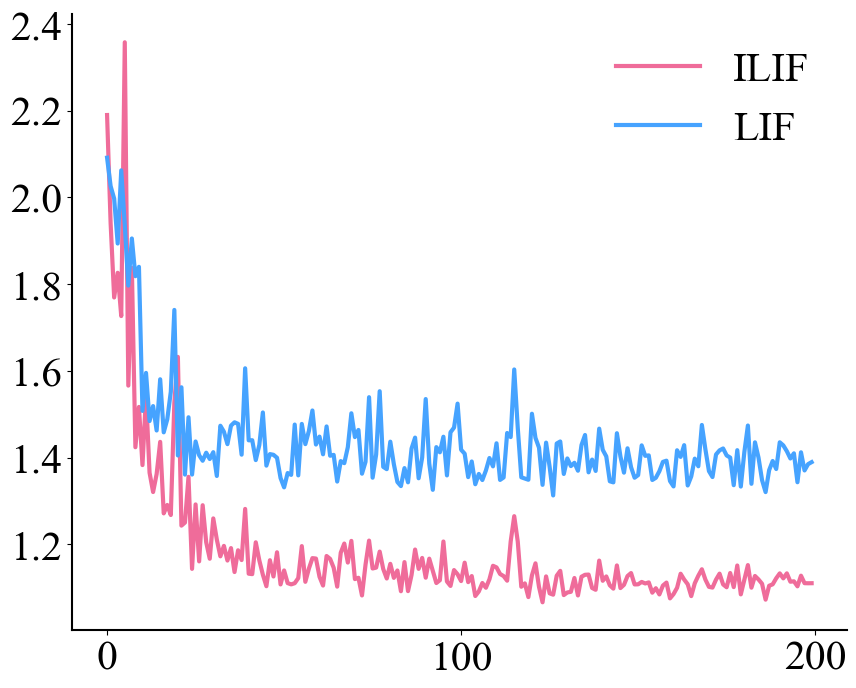}
        \vspace{-16pt}
        \caption{DVSCIFAR10}
        \label{fig:dvs_cifar10_loss}
    \end{subfigure}
    \hfill
    \begin{subfigure}[b]{0.24\textwidth}
        \centering
        \includegraphics[width=\textwidth]{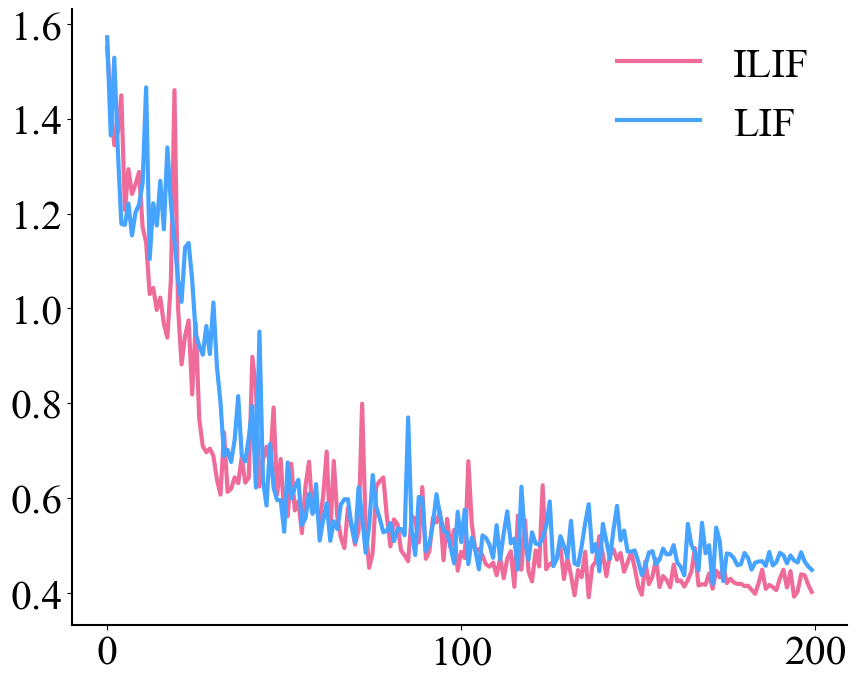}
        \vspace{-16pt}
        \caption{DVSGesture}
        \label{fig:dvs_gesture_loss}
    \end{subfigure}
    \caption{Loss curve comparison across different datasets}
    \label{fig:loss_comparison}
\end{figure}

\section{Comparison with SOTA}
\label{app:additional_experiment}
Due to space limitations, additional experiments with mainstream architectures are presented in this section, demonstrating the excellent performance of our method across various network architectures, as shown in Table~\ref{table:accuracy_app}. Models with higher accuracy achieved through data augmentation techniques are excluded from comparison. Notably, our method surpasses the traditional accuracy benchmarks on DVSCIFAR10 and DVSGesture datasets, achieving superior performance with only 200 iterations (compared to the conventional 300 iterations).

\begin{table*}[h]
    \centering
    \caption{Accuracy Comparison on CIFAR10, CIFAR100, DVSCIFAR10 and DVSGesture}
    \label{table:accuracy_app}
    \begin{threeparttable}
    \renewcommand{\arraystretch}{1.05} 
    \begin{tabular}{c| l c c c}
        \toprule
        \textbf{Dataset} & \textbf{Method} & \textbf{Network Architecture} & \textbf{Time Step} & \textbf{Accuracy (\%)} \\
        \midrule
        \multirow{12}{*}{CIFAR10} 
        & STBP-tdBN \cite{STBP-tdBN} & ResNet-19 & 4 / 6 & 92.92 / 93.16 \\
        & SEW-ResNet \cite{SEW-ResNet} & SEW-ResNet-18 & 4 & 91.22 / 94.39 \\
        & Dspike \cite{Dspike} & Modified ResNet-18 & 4 / 6 & 93.66 / 94.05 \\
        & TET \cite{TET} & ResNet-19 & 4 / 6 & 94.44 / 94.50 \\
        & SLTT* \cite{SLTT} & ResNet-18 & 6 & 94.44 \\
        & GLIF \cite{rw:glif} & ResNet-18 & 4 & 94.67 \\
        & SML \cite{SML} & ResNet-18 & 6 & 95.12 \\
        & CLIF \cite{huang2024clif} & ResNet-18 & 4 / 6 & 94.89 / 95.41 \\
        \noalign{\global\arrayrulewidth=0.2pt}
        & \textbf{Ours} & ResNet-18 & 4 / 6 & \textbf{95.24 / 95.49} \\
        \noalign{\global\arrayrulewidth=0.2pt}\cline{2-5}\noalign{\global\arrayrulewidth=0.5pt} 
        & Temporal Pruning \cite{temporal-pruning} & \multirow{3}{*}{VGG-16} & 5 & {93.44} \\
        & SLTT* \cite{SLTT} &  & 6 & {93.28} \\
        & \textbf{Ours} &  & 6 & \textbf{94.25} \\
        
        \midrule
        \multirow{9}{*}{CIFAR100} 
        & STBP-tdBN \cite{STBP-tdBN} & ResNet-19 & 4 / 6 & 70.86 / 71.12 \\
        & TET \cite{TET} & ResNet-19 & 4 / 6 & 74.74 / 74.72 \\
        & Dspike \cite{Dspike} & Modified ResNet-18 & 4 / 6 & 73.35 / 74.24 \\
        & SLTT* \cite{SLTT} & ResNet-18 & 6 & {74.38} \\
        & GLIF \cite{rw:glif} & ResNet-18 & 4 / 6 & 76.42 / 77.28 \\
        \noalign{\global\arrayrulewidth=0.2pt}
        & \textbf{Ours} & ResNet-18 & 4 / 6 & \textbf{77.43 / 78.51} \\
        \noalign{\global\arrayrulewidth=0.2pt}\cline{2-5}\noalign{\global\arrayrulewidth=0.5pt} 
        & Temporal Pruning \cite{temporal-pruning} & \multirow{3}{*}{VGG-16} & 5 & {71.58} \\
        & SLTT* \cite{SLTT} &  & 6 & {72.55} \\
        & \textbf{Ours} &  & 6 & \textbf{75.25} \\
        
        \midrule
        \multirow{9}{*}{DVSCIFAR10}
        & STBP-tdBN \cite{STBP-tdBN} & \multirow{4}{*}{ResNet-18} & 10 & 67.80 \\
        & Dspike \cite{Dspike} &  & 10 & 75.40 \\
        & InfLoR \cite{InfLoR} &  & 10 & 75.50 \\
        & \textbf{Ours} &  & 10 & \textbf{77.71} \\
        \noalign{\global\arrayrulewidth=0.2pt}\cline{2-5}\noalign{\global\arrayrulewidth=0.5pt}
        \noalign{\global\arrayrulewidth=0.2pt}\cline{2-5}\noalign{\global\arrayrulewidth=0.5pt}
        & OTTT* \cite{OTTT} & \multirow{4}{*}{VGG-11} & 10 & 76.27 \\
        & DSR \cite{DSR} &  & 20 & 77.27 \\
        & SLTT \cite{SLTT} &  & 10 & 77.17 \\
        & TET \cite{TET} &  & 10 & 77.33 \\
        & \textbf{Ours} &  & 10 & \textbf{78.60} \\
        
        \midrule
        \multirow{6}{*}{DVSGesture}
        & STBP-tdBN \cite{STBP-tdBN} & \multirow{2}{*}{ResNet-18} & 40 & 96.87 \\
        & \textbf{Ours} &  & 20 & \textbf{96.88} \\
        \noalign{\global\arrayrulewidth=0.2pt}\cline{2-5}\noalign{\global\arrayrulewidth=0.5pt}
        \noalign{\global\arrayrulewidth=0.2pt}\cline{2-5}\noalign{\global\arrayrulewidth=0.5pt}
        & OTTT \cite{OTTT} & \multirow{4}{*}{VGG-11} & 20 & 96.88 \\
        & SLTT \cite{SLTT} &  & 20 & 97.92 \\
        & CLIF \cite{huang2024clif} &  & 20 & 97.92 \\
        & \textbf{Ours} &  & 20 & \textbf{97.92} \\
        \bottomrule
    \end{tabular}
    \begin{tablenotes}
        \item[*] indicates that data augmentation was used.
    \end{tablenotes}
    \end{threeparttable}
\end{table*}

\twocolumn 
\section*{Ethical Statement}

There are no ethical issues.

\section*{Acknowledgments}
The work was supported in part by start-up funds with No. MSRI8001004 and No. MSRI9002005.

\bibliographystyle{named}
\bibliography{ijcai25}

\end{document}